\documentclass{article}

\usepackage[final]{neurips_2026}

\usepackage[utf8]{inputenc}
\usepackage[T1]{fontenc}
\usepackage{hyperref}
\usepackage{url}
\usepackage{booktabs}
\usepackage{amsfonts}
\usepackage{amsmath}
\usepackage{amssymb}
\usepackage{nicefrac}
\usepackage{microtype}
\usepackage{xcolor}
\usepackage{graphicx}
\usepackage{enumitem}
\usepackage{amsthm}
\usepackage{natbib}
\usepackage{algorithm}
\usepackage{algorithmic}
\usepackage{tikz}
\usepackage{xcolor}
\usepackage[table]{xcolor} 
\usetikzlibrary{shadows}
\usetikzlibrary{shapes,arrows,arrows.meta,positioning,fit,backgrounds,calc,patterns,decorations.pathreplacing}
\usepackage{pgfplots}
\pgfplotsset{compat=1.18}
\usepgfplotslibrary{colorbrewer,fillbetween}
\newtheorem{theorem}{Theorem}
\newtheorem{corollary}[theorem]{Corollary}
\newtheorem{definition}{Definition}

\title{\textsc{Prism}: Generation-Time Detection and Mitigation of Secret Leakage in Multi-Agent LLM Pipelines}
\author{%
  Riya Tapwal \\
  School of Computing and Electrical Engineering \\
  Indian Institute of Technology, Mandi (IIT) \\
  \texttt{riya@iitmandi.ac.in}
  \And
  Abhishek Kumar \\
  The Alan Turing Institute, London \\
  \texttt{akumar@turing.ac.uk}
  \And
  Carsten Maple \\
  University of Warwick \\
  \texttt{cm@warwick.ac.uk}
}

\begin{document}

\maketitle

\begin{abstract}
Multi-agent LLM systems introduce a security risk in which sensitive information accessed by one agent can propagate through shared context and reappear in downstream outputs, even without explicit adversarial intent. We formalise this phenomenon as \emph{propagation amplification}, where leakage risk increases across agent boundaries as sensitive content is repeatedly exposed to downstream generators. Existing defences, including prompt-based safeguards, static pattern matching, and LLM-as-judge filtering, are not designed for this setting: they either operate after generation, rely primarily on surface-form patterns, or add substantial latency without modelling the generation process itself. To resolve these issues, we propose \textsc{Prism}, a real-time defence that treats credential leakage as a sequential risk accumulation problem during generation. At each decoding step, \textsc{Prism} combines 16 signals spanning lexical, structural, information-theoretic, behavioural, and contextual features into a calibrated risk score, enabling per-token intervention through green, yellow, and red risk zones. Our central observation is that credential reproduction is often preceded by a measurable shift in generation dynamics, characterised by entropy collapse and increasing logit concentration. When combined with text-structural cues such as identifier-pattern detection, these temporal signals provide an early warning of leakage before a secret is fully reconstructed. Across a 2{,}000-task adversarial benchmark covering 13 attack categories and three pressure levels in a heterogeneous four-agent pipeline, \textsc{Prism} achieves $F_1=0.832$ with precision $=1.000$ and recall $=0.712$, while producing \textbf{no observed leakage} on our benchmark (0.0\% task-level leak rate) and preserving output utility of 0.893. It substantially outperforms the strongest baseline, Span Tagger, which achieves $F_1=0.719$ with a 15.0\% task-level leak rate.
\end{abstract}

\section{Introduction}
\label{sec:introduction}

% Consider a multi-agent assistant tasked with debugging a payment service. During the workflow, one agent invokes \texttt{read\_file(".env.production")} and retrieves a secret key. Although the key is not needed to solve the debugging task, it enters the shared context and is subsequently reproduced across intermediate outputs, including summaries, generated code, and execution logs, before appearing in the final user-facing response. This failure does not require an explicit adversarial prompt; ordinary agent collaboration is sufficient. It exposes a structural vulnerability in multi-agent LLM systems: once sensitive information enters a shared context, it can propagate across agents and re-emerge downstream.

Multi-agent LLM systems increasingly rely on shared context and tool access to coordinate complex software-engineering workflows. However, this collaboration creates a security risk: sensitive information retrieved by one agent may be unintentionally exposed to downstream agents and reproduced in later outputs. For example, a multi-agent assistant tasked with debugging a production software system may invoke \texttt{read\_file(".env.production")} and retrieve a secret key. Although the key is not needed to solve the task, it can enter the shared context and subsequently appear in intermediate summaries, generated code, execution logs, or the final user-facing response. This failure does not require an explicit adversarial prompt; ordinary agent collaboration is sufficient. It exposes a structural vulnerability in multi-agent LLM systems: once sensitive information enters a shared context, it can propagate across agents and re-emerge downstream.

Recent advances in multi-agent LLM frameworks, including AutoGen~\citep{wu2024autogen}, CAMEL~\citep{li2023camel_8}, ChatDev~\citep{inproceedings_1}, and MetaGPT~\citep{hong2023metagpt}, have enabled collaborative workflows that combine planning, retrieval, reasoning, code generation, and execution. These systems differ from single-model deployments in a critical way: information retrieved or generated by one agent is often appended to a shared, evolving context and consumed by subsequent agents. As a result, sensitive information accessed at one stage may persist across the pipeline and be reproduced by later agents, even when it is no longer relevant to the task. We formalise this phenomenon as \emph{propagation amplification}. For a pipeline of $K$ agents $\{A_1, \ldots, A_K\}$ operating over shared context $\mathcal{C}$, let a secret $s$ be retrieved at step $i$. The probability that $s$ is leaked by at least one downstream agent can be written as
\begin{equation}
P_{\text{leak}}(s) = 1 - \prod_{j=i+1}^{K} \bigl(1 - p_j(s \mid \mathcal{C}_j)\bigr),
\end{equation}
where $p_j(s \mid \mathcal{C}_j)$ denotes the probability that agent $A_j$ reproduces $s$ given its context. Under the mild lower-bound assumption $\delta = \min_j p_j(s \mid \mathcal{C}_j) > 0$, this gives $
P_{\text{leak}}(s) \geq 1 - (1-\delta)^{K-i}.
$.
Thus, even small per-agent reproduction probabilities can accumulate into substantial leakage risk as pipeline depth increases. In our undefended four-agent pipeline, 71.2\% of tasks produce at least one secret leak, with similar rates across adversarial pressure levels: 70.4\% under low pressure, 70.1\% under medium pressure, and 73.5\% under high pressure. This suggests that leakage is not only a prompt-level attack phenomenon, but also a structural consequence of shared-context agent composition. This risk is consistent with prior evidence on credential exposure and LLM privacy failures. Large-scale studies have shown that secrets frequently appear in real-world code repositories~\citep{inproceedings_3}, while recent benchmarks identify persistent leakage risks in LLM-based systems~\citep{zhou2025lessleak}. Broader surveys on LLM security and privacy~\citep{aguilera2025llm,kibriya2024privacy_24,chen2025privacy_25} also highlight sensitive information exposure as a continuing challenge, particularly in systems with complex interaction patterns. Multi-agent systems further amplify this concern: recent work such as AgentLeak~\citep{yagoubi2026agentleak} shows that internal communication channels can expose sensitive information beyond final user-facing outputs. Existing defences are not well matched to this setting. Prompt-based safeguards attempt to constrain model behaviour through instructions~\citep{inproceedings_4,wei2023jailbroken_5,mehrotra2024tree_18}, but can degrade under multi-hop reasoning and adversarial prompting. Static scanners such as detect-secrets~\citep{yelp2018detectsecrets} and TruffleHog~\citep{trufflehog2023} rely primarily on surface-level patterns and are not designed to capture contextual reuse or partial reconstruction of secrets in free-form LLM outputs. LLM-as-judge methods~\citep{NEURIPS2023_91f18a12_6} and guardrail systems such as Llama Guard~\citep{inan2023llamaguard} and LlamaFirewall~\citep{meta2025llamafirewall} can provide useful filtering, but they typically operate after generation and target broad safety categories rather than credential-specific leakage. In contrast, secret leakage in multi-agent pipelines unfolds as a sequential process during generation, requiring defences that can intervene before sensitive content is fully reconstructed and propagated.

We propose \textsc{Prism} (\textbf{P}redictive \textbf{R}isk \textbf{I}ntervention for \textbf{S}ecret \textbf{M}onitoring), a real-time defence that models credential leakage as a sequential risk accumulation process during generation. At each decoding step, \textsc{Prism} aggregates 16 heterogeneous signals spanning lexical, structural, information-theoretic, behavioural, and contextual features into a calibrated risk score,$
r_t = \sigma(\mathbf{w}^\top \mathbf{f}_t + b),
$, which supports per-token intervention through graduated risk zones during streaming generation. The key observation underlying \textsc{Prism} is that credential reproduction is often preceded by a measurable shift in generation dynamics: as a model moves from open-ended generation toward deterministic reproduction of a structured secret, token-level entropy decreases and probability mass becomes increasingly concentrated. By combining these temporal signals with text-structural cues, such as identifier-pattern and credential-style features, \textsc{Prism} detects leakage before a secret is fully reconstructed. %Empirically, the highest-weighted features include both structural and temporal signals: \texttt{identifier\_pattern} ($w=0.883$), \texttt{divergence} ($w=0.778$), and \texttt{low\_entropy} ($w=0.554$), supporting the complementarity of the multi-signal design.

\medskip
\noindent\textbf{Contributions:}
\begin{enumerate}
    \item We formalise \emph{propagation amplification} as a structural vulnerability in multi-agent LLM systems and provide a probabilistic model explaining how leakage risk compounds across agent boundaries.

    \item We introduce \textsc{Prism}, a real-time, multi-signal defence that performs per-token risk estimation and graduated intervention during generation, enabling low-latency mitigation of credential leakage.

    \item We identify a complementary set of detection signals, including entropy collapse and logit concentration from temporal generation dynamics, together with identifier-pattern and credential-style structure from generated text, and show that neither signal class alone achieves the discriminative power of the full feature set.

    \item We construct a 2{,}000-task benchmark for multi-agent secret leakage across 13 attack categories and three adversarial pressure levels, enabling large-scale empirical comparison across ten defence methods. On this benchmark, we observe that generation-time monitoring achieves $F_1=0.832$ with \textbf{no observed leakage} on our benchmark (0.0\% task-level leak rate) and utility of 0.893, outperforming the strongest baselines, including Span Tagger ($F_1=0.719$, 15.0\% task-level leak rate) and GBT Classifier ($F_1=0.684$, 19.2\% task-level leak rate).
\end{enumerate}
%We evaluate \textsc{Prism} on a 2{,}000-task adversarial benchmark spanning 13 attack categories and three pressure levels in a heterogeneous four-agent pipeline (LLaMA~3.1~8B $\to$ Gemma~2~9B $\to$ Qwen~2.5~7B $\to$ Mistral~7B). Across the full benchmark, \textsc{Prism} achieves $F_1 = 0.832$ (precision $=1.000$, recall $=0.712$) with a \textbf{0\% secret leak rate} and output utility of 0.893, substantially outperforming all baselines. The strongest competitor, Span Tagger, achieves $F_1 = 0.719$ but still leaks on 15.0\% of tasks (secret-level: 1.0\%). No existing baseline achieves zero leakage. We attribute \textsc{Prism}'s performance to strong separability between leakage and non-leakage states in temporal generation dynamics, while noting that robustness to obfuscation and adaptive adversaries remains an important direction for future work.

%% , , , , , , , , , , , , , , , , , , , , , 
%%  SECTION 2 ,  Related Work
%% , , , , , , , , , , , , , , , , , , , , , 
\section{Related Work}
\label{sec:related}

Research on LLM security and secret leakage has expanded rapidly, with surveys such as \citet{aguilera2025llm} and \citet{kibriya2024privacy_24} highlighting the growing attack surface. We focus on work most relevant to credential leakage in agentic pipelines.

\paragraph{Credential leakage and static detection:}
Early work documented widespread credential exposure in public repositories~\citep{inproceedings_3}, motivating static scanners such as detect-secrets~\citep{yelp2018detectsecrets} and TruffleHog~\citep{trufflehog2023}, which rely primarily on entropy thresholds and regex-based patterns. Subsequent work has studied leakage in LLM-based systems, including LessLeak-Bench~\citep{zhou2025lessleak} and prompt-based extraction of embedded secrets~\citep{agarwal2024prompt_26}. While these studies establish the prevalence of secret leakage, they largely operate on static code, benchmark-level data exposure, or single-model interactions, and therefore do not directly address leakage propagation in multi-agent pipelines. Related work on adversarial prompting further shows that LLM safeguards remain brittle: large-scale studies and benchmarks~\citep{inproceedings_4, wei2023jailbroken_5, 10.5555/3737916.3739661_16} demonstrate that safety alignment can be bypassed through jailbreak attacks. Automated attack methods such as Tree-of-Attacks~\citep{mehrotra2024tree_18} and universal adversarial suffixes~\citep{zou2023universal} further illustrate the persistence of these failures, while recent work~\citep{wang2025promptsafeinvestigatingprompt_20} confirms that prompt injection remains effective even against safety-trained models.

\paragraph{Safety classifiers and LLM-based filtering:}
LLM-as-judge frameworks~\citep{NEURIPS2023_91f18a12_6} and safety classifiers such as Llama Guard~\citep{inan2023llamaguard} and LlamaFirewall~\citep{meta2025llamafirewall} apply filtering or moderation to generated content. Systems such as ControlNET~\citep{yao2025controlnet} and NeMo Guardrails~\citep{rebedea2023nemo} extend this paradigm to more structured LLM workflows. However, these approaches primarily target broad content-safety categories rather than credential-specific leakage, and they typically operate after generation has already occurred. This makes them less suitable for multi-agent settings in which a secret may enter shared context and propagate before final-output filtering is applied. Code-specific defences~\citep{10.1145/3661167.3661263_17} highlight the need for domain-specific analysis, but do not model generation-time leakage dynamics.

\paragraph{Multi-agent systems and propagation:}
Multi-agent frameworks such as ChatDev~\citep{inproceedings_1}, MetaGPT~\citep{hong2023metagpt}, AutoGen~\citep{wu2024autogen}, CAMEL~\citep{li2023camel_8}, CrewAI~\citep{crewai2024_9}, and HuggingGPT~\citep{shen2023hugginggpt} enable collaborative workflows through inter-agent communication, tool use, and shared intermediate state. These capabilities also introduce new attack surfaces, since information produced or retrieved by one agent can become visible to downstream agents. Subsequent work explores agent capabilities and behaviours~\citep{liu2024agentbench_11, wang2023voyager, chen2024agentverse_14, 10.1145/3586183.3606763_15}, but generally does not treat credential leakage as a central threat model. AgentLeak~\citep{yagoubi2026agentleak} provides closer evidence that sensitive information can appear in internal agent communications and leak across channels. While that work focuses on broader privacy leakage, our setting targets immediately exploitable credentials and emphasises intervention during generation, before sensitive content is fully reconstructed or propagated.

\paragraph{Static scanners and non-LLM detection:}
Static tools such as detect-secrets~\citep{yelp2018detectsecrets}, TruffleHog~\citep{trufflehog2023}, and Presidio~\citep{microsoft2019presidio} are designed for structured code, secrets, and PII detection. Applying them to free-form LLM outputs requires substantial recalibration because generated text differs from source code in token distribution, formatting regularity, and contextual reuse of sensitive content. Classical ML approaches, such as gradient-boosted classifiers and span taggers, provide useful baselines, but they operate on completed outputs and lack access to generation-time dynamics. Recent work also emphasises the need for structured evaluation of LLM risks~\citep{article_31} and demonstrates persistent failures in safety-critical domains~\citep{hung-etal-2023-walking_34}. Surveys of privacy risks~\citep{chen2025privacy_25} further highlight the gap between broad content-safety moderation and credential-specific protection.

Overall, existing approaches remain limited for this setting because they are typically (i) static or post-hoc, (ii) designed primarily for single-model interactions, or (iii) focused on broad safety categories rather than credential-specific leakage. In contrast, \textsc{Prism} provides a generation-time defence tailored to multi-agent pipelines, explicitly modelling propagation risk and enabling intervention before secrets are fully reconstructed.

\section{Threat Model and Problem Formulation}
\label{sec:threat}

\subsection{System Model}

We consider a sequential multi-agent pipeline 
$\mathcal{A}=\langle A_1,\ldots,A_K\rangle$, where each agent $A_k$ may be instantiated from a different LLM and performs a functional role such as planning, retrieval, code generation, or execution. Agents communicate through an append-only shared context:
\[
\mathcal{C}_k = \mathcal{C}_{k-1} \,\|\, o_{k-1}, 
\qquad 
\mathcal{C}_1 = (\texttt{task},\,\texttt{ctx}),
\]
where $o_{k-1}$ denotes the output produced by agent $A_{k-1}$ and $\|$ denotes context concatenation. Each agent has access to a tool set $\mathcal{G}$, such as \texttt{read\_file}, \texttt{db\_query}, and \texttt{api\_call}, connected to a shared repository $\mathcal{R}$ containing secrets $\mathcal{S}=\{s_1,\ldots,s_M\}$. If a secret is retrieved or generated by an upstream agent, it may be appended to the shared context and thereby become available to all downstream agents.

\subsection{Adversary Model}

We assume a realistic black-box adversary who controls only the input prompt $\texttt{task}$ and has no access to model parameters, tool implementations, or repository contents. This captures a malicious or compromised user interacting with an otherwise trusted multi-agent system. Leakage may arise in three forms: \emph{direct leakage}, where an agent retrieves and emits a secret; \emph{propagation leakage}, where downstream agents reproduce a secret previously introduced into the shared context; and \emph{adversarial extraction}, where prompt-level manipulation increases the likelihood of disclosure through coercive cues. We model adversarial pressure using a transformation $\phi(t,p)$, where $p \in \{\textsc{Low},\textsc{Medium},\textsc{High}\}$ controls the strength of such cues, including urgency, authority, or explicit requests for sensitive information.

\subsection{Leakage Process}

\begin{definition}[Secret Leakage]
A task $t$ produces leakage if the final output $o_K$ contains any secret $s \in \mathcal{S}$:
\[
\textsc{Leak}(t) = \mathbb{1}[\exists\, s \in \mathcal{S} : s \sqsubseteq o_K].
\]
\end{definition}

Leakage is sequential in nature. Once a secret enters the shared context, it may be reproduced by one or more downstream agents. If a secret $s$ is introduced at step $i$, its probability of being leaked by at least one downstream agent is
\[
P_{\text{leak}}(s) = 1 - \prod_{j=i+1}^{K} (1 - p_j(s \mid \mathcal{C}_j)),
\]
where $p_j(s \mid \mathcal{C}_j)$ denotes the probability that agent $A_j$ reproduces $s$ given its context. Under the mild lower-bound assumption $p_j(s \mid \mathcal{C}_j) \geq \delta > 0$, this yields
\[
P_{\text{leak}}(s) \geq 1 - (1-\delta)^{K-i}.
\]
Thus, even small per-agent reproduction probabilities can accumulate into substantial overall leakage risk as pipeline depth increases. We refer to this compounding effect as \emph{propagation amplification}.

\noindent\textit{Remark on independence:}
The product formulation above corresponds to the conditional-independence case, where downstream reproduction events are treated as independent conditioned on their contexts. In practice, this assumption is conservative for append-only multi-agent pipelines. If agent $A_j$ reproduces $s$, the secret becomes part of $\mathcal{C}_{j+1}$, which can increase $p_{j+1}(s \mid \mathcal{C}_{j+1})$ for subsequent agents. Such positive dependence strengthens the qualitative conclusion that leakage risk is non-decreasing with pipeline depth. Equivalently, the Bonferroni union bound,
\[
P_{\text{leak}}(s) \leq \sum_j p_j(s \mid \mathcal{C}_j),
\]
and corresponding inclusion--exclusion terms support the same intuition: as additional downstream agents are exposed to the secret, the opportunity for reproduction increases. Empirically, we validate this effect by varying pipeline depth $K \in \{1,2,3,4\}$ and observing a monotone increase in task-level leak rate.

\subsection{Defence Objective}

A defence $\mathcal{D}$ maps each intermediate output $o_k$ to a filtered output $\hat{o}_k$. The objective is to maximise leakage detection while limiting unnecessary modification of benign outputs:
\[
\max_{\mathcal{D}} \; F_1(\mathcal{D}, \mathcal{S}) 
\quad \text{s.t.} \quad 
\text{OverBlock}(\mathcal{D}) \leq \epsilon.
\]
Here, $F_1(\mathcal{D}, \mathcal{S})$ measures detection performance over protected secrets, and $\text{OverBlock}(\mathcal{D})$ denotes the fraction of benign tasks whose outputs are unnecessarily modified. We define task-level utility as
\[
\mathcal{U}(\mathcal{D}) = 
\frac{|\{t \in \mathcal{B} : \hat{o}_t = o_t\}|}{|\mathcal{B}|},
\]
where $\mathcal{B}$ denotes the set of evaluated tasks. This utility measure captures the fraction of tasks whose outputs pass through the defence unmodified. It is conservative: even when a task is modified, most non-sensitive tokens may still be preserved. We therefore distinguish task-level over-blocking, measured as false positives on clean tasks, from instance-level utility degradation within outputs that are only partially sanitised.

\section{The \textsc{Prism} Framework}
\label{sec:prism}

\subsection{Overview}

We propose \textsc{Prism}, a real-time defence that detects and mitigates secret leakage by monitoring the \emph{dynamics of token generation}. The approach is motivated by a simple observation: credential leakage is not an instantaneous event, but a \emph{sequential process}. As a model shifts from open-ended generation toward deterministic reproduction of a structured or memorised string, such as an API key, its token-level uncertainty decreases and probability mass becomes increasingly concentrated. This transition can provide an early signal of leakage before the full secret is emitted. \textsc{Prism} operationalises this insight by treating generation as sequential risk accumulation. At each decoding step $t$, it estimates a scalar risk score $r_t \in [0,1]$ from observable generation signals and applies intervention before the token is committed to the shared context. This design allows \textsc{Prism} to act before full secret reconstruction, reducing the likelihood that sensitive information propagates to downstream agents. Algorithm~\ref{alg:prism} summarises the per-token monitoring process: for each token, features are extracted, risk is evaluated, and the output is either passed through, sanitised, or halted according to the estimated risk. To complement generation-time monitoring, \textsc{Prism} also includes a lightweight post-generation audit based on hashed 8-gram matching, denoted ZK-RC. For each registered secret, character-level 8-grams are hashed with SHA-256 and stored. At audit time, the output is converted into 8-grams, hashed, and compared against this protected set. This allows \textsc{Prism} to check for residual protected content without storing secrets in plaintext. ZK-RC acts as a fallback layer for rare cases in which individual tokens remain below the intervention threshold but collectively reconstruct a secret. The primary defence, however, remains generation-time detection and intervention. Figure~\ref{fig:architecture} illustrates how \textsc{Prism} is integrated into the multi-agent pipeline. Each agent is monitored by a dedicated \textsc{Prism} instance before its output is written to the shared context buffer, limiting the propagation of sensitive content to downstream agents.

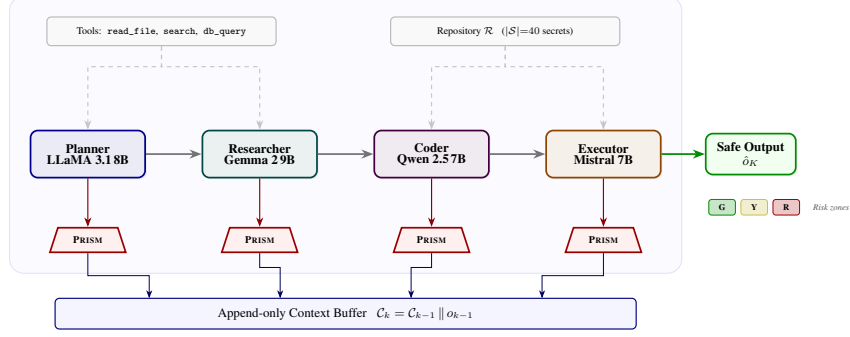
\begin{figure}[t]
\centering
\resizebox{0.8\textwidth}{!}{%
\begin{tikzpicture}[
    >=Stealth,
    agent/.style={
      rectangle, draw=#1, fill=#1!8, very thick,
      minimum width=2.8cm, minimum height=1.15cm, align=center,
      rounded corners=5pt, font=\small\bfseries,
      drop shadow={shadow xshift=0.4mm, shadow yshift=-0.4mm, fill=black!12}},
    agent/.default={blue!55!black},
    prism/.style={
      trapezium, trapezium angle=72, draw=red!60!black, fill=red!6,
      very thick, minimum width=1.6cm, minimum height=0.6cm,
      align=center, font=\scriptsize\bfseries, inner sep=3pt},
    resource/.style={
      rectangle, draw=gray!55, fill=gray!5, thick,
      minimum width=5.6cm, minimum height=0.7cm, align=center,
      rounded corners=3pt, font=\scriptsize},
    ctxbuf/.style={
      rectangle, draw=blue!50!black, fill=blue!4, thick,
      minimum width=14.2cm, minimum height=0.75cm, align=center,
      rounded corners=3pt, font=\small},
    zone/.style={
      rectangle, minimum width=0.65cm, minimum height=0.38cm,
      align=center, font=\tiny\bfseries, rounded corners=2pt,
      draw=#1, fill=#1!22},
    output/.style={
      rectangle, draw=green!55!black, fill=green!6, very thick,
      minimum width=2.2cm, minimum height=1.0cm, align=center,
      rounded corners=5pt, font=\small\bfseries,
      drop shadow={shadow xshift=0.4mm, shadow yshift=-0.4mm, fill=black!12}},
    arr/.style={->, very thick, draw=black!55},
    arrdash/.style={->, thick, dashed, draw=gray!45},
    arrprism/.style={->, thick, draw=red!55!black},
    ctxarr/.style={->, thin, draw=blue!35!black},
]

\node[resource] (tools) at (1.8, 3.0)
  {Tools:\; \texttt{read\_file},\; \texttt{search},\; \texttt{db\_query}};
\node[resource] (repo) at (10.2, 3.0)
  {Repository $\mathcal{R}$\;\; ($|\mathcal{S}|{=}40$ secrets)};

\node[agent={blue!55!black}]    (planner)    at (0, 0)
  {Planner\\[-2pt]{\footnotesize LLaMA 3.1\,8B}};
\node[agent={teal!55!black}]    (researcher) at (4.2, 0)
  {Researcher\\[-2pt]{\footnotesize Gemma 2\,9B}};
\node[agent={violet!50!black}]  (coder)      at (8.4, 0)
  {Coder\\[-2pt]{\footnotesize Qwen 2.5\,7B}};
\node[agent={orange!55!black}]  (executor)   at (12.6, 0)
  {Executor\\[-2pt]{\footnotesize Mistral 7B}};

\node[prism] (p1) at (0,    -2.1) {\textsc{Prism}};
\node[prism] (p2) at (4.2,  -2.1) {\textsc{Prism}};
\node[prism] (p3) at (8.4,  -2.1) {\textsc{Prism}};
\node[prism] (p4) at (12.6, -2.1) {\textsc{Prism}};

\node[ctxbuf] (ctx) at (6.3, -3.9)
  {Append-only Context Buffer\;\;
   $\mathcal{C}_k = \mathcal{C}_{k-1}\,\|\,o_{k-1}$};

\node[output] (out) at (16.2, 0) {Safe Output\\$\hat{o}_K$};

\node[zone=green!55!black] (zg) at (15.5, -1.3) {G};
\node[zone=yellow!75!black, right=0.12cm of zg] (zy) {Y};
\node[zone=red!65!black, right=0.12cm of zy] (zr) {R};
\node[font=\tiny\itshape, gray!70!black, right=0.2cm of zr] {Risk zones};

\draw[arr] (planner)    -- (researcher);
\draw[arr] (researcher) -- (coder);
\draw[arr] (coder)      -- (executor);
\draw[arr, draw=green!55!black] (executor) -- (out);

\draw[arrprism] (planner.south)    -- (p1.north);
\draw[arrprism] (researcher.south) -- (p2.north);
\draw[arrprism] (coder.south)      -- (p3.north);
\draw[arrprism] (executor.south)   -- (p4.north);

\draw[arrdash] (tools.south) -- ++(0,-0.5) -| (planner.north);
\draw[arrdash] (tools.south) -- ++(0,-0.5) -| (researcher.north);
\draw[arrdash] (repo.south)  -- ++(0,-0.5) -| (coder.north);
\draw[arrdash] (repo.south)  -- ++(0,-0.5) -| (executor.north);

\draw[ctxarr] (p1.south) -- ++(0,-0.55) -| ([xshift=-4.8cm]ctx.north);
\draw[ctxarr] (p2.south) -- ++(0,-0.35) -| ([xshift=-1.6cm]ctx.north);
\draw[ctxarr] (p3.south) -- ++(0,-0.35) -| ([xshift=1.6cm]ctx.north);
\draw[ctxarr] (p4.south) -- ++(0,-0.55) -| ([xshift=4.8cm]ctx.north);

\begin{pgfonlayer}{background}
  \node[fit=(tools)(repo)(planner)(executor)(p1)(p4),
    fill=blue!2, draw=blue!12, rounded corners=10pt,
    inner sep=14pt, behind path] {};
\end{pgfonlayer}

\end{tikzpicture}%
}
\caption{\textsc{Prism} system architecture. Each agent in the four-stage pipeline is monitored by an independent \textsc{Prism} instance that evaluates generation-time risk before the agent's output enters the shared context buffer. Dashed arrows denote tool access to the repository $\mathcal{R}$ containing secrets $\mathcal{S}$.}
\label{fig:architecture}
\end{figure}

\subsection{Sequential Risk Formulation}

Given a partial output $x_{1:t}$ and next-token distribution $\ell_t$, \textsc{Prism} computes a feature representation $\mathbf{f}_t$ and evaluates
\begin{equation}
r_t = \sigma(\mathbf{w}^\top \mathbf{f}_t + b),
\end{equation}
where $\sigma(\cdot)$ denotes the logistic sigmoid. Rather than relying only on surface-form matching, \textsc{Prism} combines two complementary classes of signals.

The first class captures \emph{temporal generation dynamics}:
\begin{itemize}[leftmargin=*]
    \item \textbf{Entropy collapse:} token-level entropy decreases as the model becomes increasingly certain about the next token.
    \item \textbf{Logit concentration:} probability mass concentrates on a narrow set of tokens, indicating movement toward deterministic reproduction.
    \item \textbf{Trajectory escalation:} the fraction of generation steps with increasing risk grows as the sequence converges toward a credential-like pattern.
\end{itemize}

The second class captures \emph{text-structural signals} observable from the generated sequence:
\begin{itemize}[leftmargin=*]
    \item \textbf{Identifier-pattern detection:} the output contains subsequences matching credential-like syntax, such as key-value assignments or alphanumeric tokens of credential-like length.
    \item \textbf{Credential style and numeric runs:} the output contains structured, high-entropy token sequences characteristic of secrets.
\end{itemize}

Empirically, both signal classes contribute to detection. The highest-weighted features in the trained classifier include both structural and temporal signals: \texttt{identifier\_pattern} ($w{=}0.883$), \texttt{divergence} ($w{=}0.778$), and \texttt{low\_entropy} ($w{=}0.554$). Neither class alone achieves the discriminative power of the full feature set, indicating that the effectiveness of \textsc{Prism} arises from their complementarity.

\subsection{Intervention as Control}

The risk score is mapped to a three-level control policy:
\begin{equation}
\text{Zone}(r_t) =
\begin{cases}
\textsc{Green} & r_t < \tau_1, \\
\textsc{Yellow} & \tau_1 \leq r_t < \tau_2, \\
\textsc{Red} & r_t \geq \tau_2.
\end{cases}
\end{equation}

\noindent
In the \textsc{Green} zone, generation proceeds unchanged. In the \textsc{Yellow} zone, local sanitisation is applied to the risky token or span. In the \textsc{Red} zone, generation is halted before the secret is completed. This formulation treats leakage mitigation as a \emph{control problem} over the generation process rather than as a post-hoc filtering task.

\begin{algorithm}[t]
\caption{\textsc{Prism}: Real-Time Leakage Monitoring}
\label{alg:prism}
\begin{algorithmic}[1]
\REQUIRE Token stream $\{x_t\}$ with log-probabilities $\{\ell_t\}$, parameters $(\mathbf{w}, b)$, thresholds $(\tau_1, \tau_2)$
\ENSURE Filtered output $\hat{o}$

\STATE $\hat{o} \leftarrow \emptyset$
\STATE $\textit{trajectory} \leftarrow [\,]$

\FOR{each token $x_t$}
    \STATE $\mathbf{f}_t \leftarrow \textsc{ExtractFeatures}(x_{1:t}, \ell_t)$
    \STATE $r_t \leftarrow \sigma(\mathbf{w}^{\top}\mathbf{f}_t + b)$
    \STATE $\textit{trajectory}.\textsc{Append}(r_t)$

    \IF{$r_t < \tau_1$}
        \STATE $\hat{o} \leftarrow \hat{o} \,\Vert\, x_t$ \COMMENT{\textsc{Green}: pass through}
    \ELSIF{$r_t < \tau_2$}
        \STATE $\hat{o} \leftarrow \hat{o} \,\Vert\, \textsc{Sanitise}(x_t)$ \COMMENT{\textsc{Yellow}: sanitise}
    \ELSE
        \STATE $\hat{o} \leftarrow \hat{o} \,\Vert\, \texttt{[REDACTED]}$ \COMMENT{\textsc{Red}: halt}
        \STATE \textbf{break}
    \ENDIF
\ENDFOR

\RETURN $\hat{o}$
\end{algorithmic}
\end{algorithm}

% \begin{algorithm}[t]
% \caption{\textsc{Prism}: Real-Time Leakage Monitoring}
% \label{alg:prism}
% \begin{algorithmic}[1]
% \REQUIRE Token stream $\{x_t\}$ with log-probabilities $\{\ell_t\}$, parameters $(\mathbf{w}, b)$, thresholds $(\tau_1, \tau_2)$
% \ENSURE Filtered output $\hat{o}$

% \STATE $\hat{o} \leftarrow \emptyset$, \quad $\textit{trajectory} \leftarrow [\,]$

% \FOR{each token $x_t$}
%     \STATE $\mathbf{f}_t \leftarrow \textsc{ExtractFeatures}(x_{1:t}, \ell_t)$
%     \STATE $r_t \leftarrow \sigma(\mathbf{w}^\top \mathbf{f}_t + b)$
%     \STATE $\textit{trajectory}.\textsc{Append}(r_t)$

%     \IF{$r_t < \tau_1$} \hfill \COMMENT{\textsc{Green}: pass through}
%     \STATE $\hat{o} \leftarrow \hat{o} \,\|\, x_t$
%     \ELSIF{$r_t < \tau_2$} \hfill \COMMENT{\textsc{Yellow}: sanitise}
%         \STATE $\hat{o} \leftarrow \hat{o} \,\|\, \textsc{Sanitise}(x_t)$
%     \ELSE \hfill \COMMENT{\textsc{Red}: halt}
%         \STATE $\hat{o} \leftarrow \hat{o} \,\|\, \texttt{[REDACTED]}$
%         \STATE \textbf{break}
%     \ENDIF
% \ENDFOR

% \RETURN $\hat{o}$
% \end{algorithmic}
% \end{algorithm}

\noindent\textbf{Yellow-zone sanitisation.}
The \textsc{Sanitise} function replaces a \textsc{Yellow}-zone token with the fixed placeholder \texttt{[MASK]}. This preserves position and approximate token count for downstream agents while removing the literal token value. %In practice, fewer than 3\% of tokens across evaluated tasks fall into the \textsc{Yellow} zone; most interventions occur through \textsc{Red}-zone halts. The \textsc{Yellow} zone therefore acts as a soft buffer between normal generation and full termination, reducing abrupt truncation for borderline-risk tokens.

\subsection{Feature Realisation}

Although the core signal comes from temporal generation dynamics, \textsc{Prism} incorporates complementary lexical, structural, behavioural, and contextual features. These include:
\begin{itemize}[leftmargin=*]
    \item \textbf{Information-theoretic features:} entropy and logit concentration;
    \item \textbf{Structural patterns:} identifier formats and numeric sequences;
    \item \textbf{Verbatimness signals:} n-gram overlap and repetition;
    \item \textbf{Contextual signals:} tool usage, keyword density, and risk trajectory.
\end{itemize}

\section{Experimental Setting and Defence Paradigms}
\label{sec:setup}

\paragraph{Problem setting and Benchmark:}
We evaluate secret leakage in a multi-agent LLM pipeline in which agents collaborate through a shared context buffer and use external tools to retrieve information from a repository containing sensitive credentials. Unlike single-model settings, sensitive information introduced by one agent can be exposed to downstream agents, allowing leakage risk to propagate across the pipeline. We construct a 2{,}000-task adversarial benchmark spanning 13 attack categories and three pressure levels: low ($700$ tasks), medium ($660$ tasks), and high ($640$ tasks). The tasks simulate realistic leakage scenarios, including prompt injection, social engineering, configuration inspection, and inter-agent propagation. Each task targets a small subset of credentials while agents operate over a shared repository. Across all tasks, this yields 30{,}900 ground-truth secrets in generated outputs. For completeness and reproducibility, detailed descriptions of the benchmark construction, environment, defence paradigms, and evaluation protocol are provided in Appendix~\ref{app:setup}.

%\paragraph{Benchmark:}

\paragraph{Environment and pipeline:}
Agents interact with a simulated enterprise repository and knowledge base through six tools: \texttt{read\_file}, \texttt{search\_files}, \texttt{get\_config}, \texttt{db\_query}, \texttt{api\_call}, and \texttt{execute\_code}. We use a heterogeneous four-agent pipeline: LLaMA~3.1~8B as Planner, Gemma~2~9B as Researcher, Qwen~2.5~7B as Coder, and Mistral~7B as Executor. Agents communicate through an append-only shared context. All main results use the \textbf{white-box} version of \textsc{Prism}, which includes all 16 features, including token-level log-probabilities. Black-box operation, using seven text-only features with log-probability and hidden-state features zeroed, is evaluated separately in Appendix~\ref{app:blackbox}.

\paragraph{Defence paradigms:}
To the best of our knowledge, prior work has not directly addressed generation-time secret leakage in multi-agent LLM pipelines. Existing approaches either focus on single-model settings, such as decoding-time controls, or apply post-hoc filtering after generation, and therefore do not directly model cross-agent propagation. We compare \textsc{Prism} against representative defence paradigms adapted from these settings. These baselines operate after the full output has been generated, meaning that a secret may already have appeared in intermediate or final text before filtering is applied. In contrast, \textsc{Prism} operates during generation, estimating token-level risk in real time and intervening before a secret is fully reconstructed. Detection is driven by temporal generation signals, including trajectory trend, divergence, and entropy, and is complemented by a lightweight ZK-RC post-pass for residual fragments. All defence mechanisms are evaluated in their standard configurations without task-specific tuning. %This reflects realistic deployment settings, where such methods are typically applied out-of-the-box. Table~\ref{tab:defense-comparison-app} (Appendix~\ref{app:setup}) provides a full paradigm-by-paradigm comparison; all baselines operate post hoc, whereas \textsc{Prism} is the only method that intervenes per-token during generation.

\section{Results and Analysis}
\label{sec:results}

\subsection{Overall Defence Effectiveness}

Table~\ref{tab:combined} reports task-level leakage rate, average leaked secret instances per task, and classification metrics across all evaluated defences. The undefended baseline leaks at least one secret in 71.2\% of tasks and exposes 11.52 secret instances per task on average, confirming that leakage is both frequent and severe in the evaluated multi-agent pipeline. All defence methods achieve precision of 1.000 with zero false positives (FPR $= 0.000$), reflecting conservative blocking behaviour: each method intervenes only when it is confident. The differentiating factor across methods is therefore \emph{recall}, the fraction of leakage events that are actually caught. Rule-based and scanner-based approaches fare poorly in this regard. Presidio, TruffleHog, BasicGuardrail, and LLM Judge all remain above 63\% task-level leak rate with recall below 0.10, suggesting that surface-pattern matching and broad content-safety filtering are insufficient for detecting credential reproduction in free-form LLM outputs. Prompt-based and entropy-scanner methods (PromptInstructionDefense, detect-secrets) offer moderate improvement but still leave over 40\% of tasks exposed. Further, supervised methods perform substantially better. GBT Classifier reduces the task-level leak rate to 19.2\% (recall $= 0.520$, $F_1 = 0.684$), while Span Tagger achieves 15.0\% (recall $= 0.562$, $F_1 = 0.719$). However, both still leave considerable residual leakage, and neither models the generation process that produces leaks. However, \textsc{Prism} reduces the observed task-level leak rate to \textbf{0.0\%} with zero leaked instances, making it the only method to eliminate observed leakage on this benchmark. It achieves the highest recall ($0.712$) and $F_1$ ($0.832$) among all methods. This recall advantage is consistent with \textsc{Prism}'s use of generation-time signals: by monitoring entropy collapse and trajectory escalation during decoding, it detects the \emph{process} of credential reconstruction before a secret is fully emitted, rather than relying on surface-form matching over completed outputs. Bootstrap 95\% confidence intervals using 10k resamples show no overlap between \textsc{Prism} and any baseline, confirming that the observed improvements are statistically significant.

\begin{table}[t]
\centering
\small
\caption{Overall defence effectiveness: task-level leakage rate, average leaked instances per task, and classification metrics. All methods achieve Precision = 1.000 and FPR = 0.000; the differentiating factor is Recall.}
\label{tab:combined}
\begin{tabular}{l c c c c}
\toprule
\textbf{Method} & \textbf{Task Leak Rate (\%)} & \textbf{Leaks / Task} & \textbf{Recall} & \textbf{F1} \\
\midrule
\rowcolor{red!10}
NoFilter                 & 71.2 & 11.52 &   --  &  --   \\
Presidio                 & 65.8 & 3.01  & 0.054 & 0.102 \\
TruffleHog               & 64.2 & 5.97  & 0.070 & 0.130 \\
BasicGuardrail           & 63.1 & 3.36  & 0.081 & 0.149 \\
LLM Judge                & 63.1 & 3.36  & 0.081 & 0.149 \\
PromptInstructionDefense & 49.3 & 7.58  & 0.280 & 0.437 \\
detect-secrets           & 43.3 & 0.45  & 0.279 & 0.436 \\
GBT Classifier           & 19.2 & 2.49  & 0.520 & 0.684 \\
Span Tagger              & 15.0 & 0.15  & 0.562 & 0.719 \\
\midrule
\rowcolor{green!15}
\textbf{\textsc{Prism}}  & \textbf{0.0} & \textbf{0.00} & \textbf{0.712} & \textbf{0.832} \\
\bottomrule
\end{tabular}
\end{table}

\subsection{Complementarity of Temporal and Structural Signals}
\label{sec:ablation}

Table~\ref{tab:feature_class_main} evaluates the contribution of the two main signal classes in \textsc{Prism}: temporal generation signals and text-structural signals. We compare temporal-only \textsc{Prism}, where text-observable features are zeroed; structural-only \textsc{Prism}, where log-probability and model-internal features are zeroed; and the full combined model. Temporal signals alone leave a residual leak rate of 5.3\%, while structural signals alone leave a residual leak rate of 1.4\%. Only the combined model achieves 0.0\% observed leakage. These results indicate that both signal classes contribute complementary information: temporal dynamics help identify the onset of deterministic reproduction, while structural cues capture credential-like forms in the generated text. Full details of the evaluation populations and methodology are provided in Appendix~\ref{app:blackbox}.

\begin{table}[h]
\centering
\small
\caption{Feature-class ablation. Temporal-only and structural-only variants are evaluated on dedicated subsets ($n\!=200$); combined \textsc{Prism} is evaluated on the full 2{,}000-task benchmark. Neither signal class alone eliminates leakage.}
\label{tab:feature_class_main}
\begin{tabular}{lccc}
\toprule
\textbf{Signal class} & \textbf{Recall} & \textbf{$F_1$} & \textbf{Leak rate} \\
\midrule
Temporal only (logprob signals)  & $\approx$0.85 & $\approx$0.92 & $\geq$5.3\% \\
Structural only (text signals)   & 0.947         & 0.947         & 1.4\%  \\
\midrule
\rowcolor{green!15}
\textbf{Combined (\textsc{Prism})} & \textbf{0.712} & \textbf{0.832} & \textbf{0.0\%} \\
\bottomrule
\end{tabular}
\end{table}

\subsection{Discussion and Limitations}
\label{limit}
The results support the view that leakage in multi-agent pipelines is not merely an isolated output-level failure, but can arise from sequential propagation through shared context. Post-hoc defences can detect some leaked outputs, but they do not directly model the temporal process by which secrets are reconstructed during generation or propagated across intermediate agents. \textsc{Prism} addresses this gap by intervening during generation before secrets are fully reconstructed. However several limitations remain such as first, the benchmark is synthetic and may not capture all behaviours of production multi-agent systems. Second, the main version of \textsc{Prism} assumes white-box access to token-level log-probabilities; when log-probability features are removed, black-box operation introduces a 1.4\% residual leak rate (Appendix~\ref{app:blackbox}). A fuller discussion is provided in Appendix~\ref{app:discussion}.

\section{Conclusion}

We studied secret leakage in multi-agent LLM pipelines and identified \emph{propagation amplification} as a structural risk arising from sequential information sharing. Our analysis shows that once sensitive information enters a shared context, it can propagate across downstream agents and reappear in later outputs, even when leakage is not explicitly requested. We introduced \textsc{Prism}, a generation-time defence that monitors token-level dynamics and intervenes before secrets are fully reconstructed. By combining temporal generation signals, including entropy collapse and trajectory escalation, with text-structural cues such as identifier patterns and credential style, \textsc{Prism} detects leakage processes that are difficult to capture using purely post-hoc or surface-form methods. Across a 2{,}000-task benchmark, \textsc{Prism} achieves $F_1=0.832$, recall $=0.712$, and \textbf{no observed leakage} on our benchmark (0.0\% task-level leak rate), while preserving output utility of 89.3\%. It outperforms the strongest baselines, including Span Tagger ($F_1=0.719$, 15.0\% task-level leak rate) and GBT Classifier ($F_1=0.684$, 19.2\% task-level leak rate). These results suggest that securing agentic systems requires moving beyond static or post-hoc filtering toward continuous monitoring of the generation process itself. Future work will extend \textsc{Prism} to more adaptive leakage strategies, including obfuscated, encoded, and fragmented disclosures, and explore integration with training-time and decoding-time defences for stronger end-to-end protection.

\bibliographystyle{plainnat}
\bibliography{references}

\appendix

\section{Feature Definitions}
\label{appendix1}

\textsc{Prism} computes a 16-dimensional feature vector $\mathbf{f}_t \in \mathbb{R}^{16}$ at each generation step. Each feature is normalised to $[0,1]$ and captures a distinct signal relevant to potential secret leakage, spanning generation dynamics, structural patterns, and contextual cues.

\paragraph{Notation:}
$x_{1:t}$ denotes generated tokens, $P(\cdot)$ the next-token distribution, $\mathcal{V}$ the vocabulary, and $\epsilon=10^{-6}$.

% =========================================================
\subsection{Hidden-State Signals}

\textbf{Magnitude:} Normalised $\ell_2$ norm of the hidden state vector:
\[
f_{\text{magnitude}} = \min\!\left(\frac{\|\mathbf{h}\|_2}{2L}, 1\right).
\]

\textbf{Variance:} Normalised variance of hidden state activations:
\[
f_{\text{variance}} = \min(\mathrm{Var}(\mathbf{h}), 1).
\]

% =========================================================
\subsection{Provenance Signals}

\textbf{Public provenance:} Fraction of attention mass assigned to public sources:
\[
f_{\text{public}} = \frac{\alpha_{\text{pub}}}{\alpha_{\text{pub}} + \alpha_{\text{priv}} + \alpha_{\text{sys}} + \epsilon}.
\]

\textbf{Private provenance:} Fraction of attention mass assigned to private sources:
\[
f_{\text{private}} = \frac{\alpha_{\text{priv}}}{\alpha_{\text{pub}} + \alpha_{\text{priv}} + \alpha_{\text{sys}} + \epsilon}.
\]

\textbf{System provenance:} Fraction of attention mass assigned to system-level sources:
\[
f_{\text{system}} = \frac{\alpha_{\text{sys}}}{\alpha_{\text{pub}} + \alpha_{\text{priv}} + \alpha_{\text{sys}} + \epsilon}.
\]

% =========================================================
\subsection{Verbatimness Signals}

\textbf{Entropy:} Normalised entropy of the next-token distribution:
\[
f_{\text{entropy}} = \frac{-\sum_{v \in \mathcal{V}} P(v)\log P(v)}{\log|\mathcal{V}|}.
\]

\textbf{N-gram overlap:} Fraction of generated n-grams overlapping with source material:
\[
f_{\text{overlap}} = \frac{|\text{ngrams}_3(x_{1:t}) \cap \text{ngrams}_3(\mathcal{M})|}{|\text{ngrams}_3(x_{1:t})|}.
\]

\textbf{Repetition:} Indicator of token repetition within a local context window:
\[
f_{\text{repetition}} = \mathbb{1}[x_{t+1} \in x_{t-4:t}].
\]

% =========================================================
\subsection{Structural and Semantic Signals}

\textbf{Identifier pattern:} Indicator of structured identifiers (e.g., email, UUID, IP):
\[
f_{\text{identifier}} = \mathbb{1}[\text{identifier regex matches}].
\]

\textbf{Numeric runs:} Indicator of long digit sequences (length $\geq 8$):
\[
f_{\text{numeric}} = \mathbb{1}[\text{digit-run length} \geq 8].
\]

\textbf{Credential style:} Indicator of credential-like patterns (e.g., API keys, tokens):
\[
f_{\text{credential}} = \mathbb{1}[\text{credential regex or keywords}].
\]

% =========================================================
\subsection{Distributional Signals}

\textbf{Divergence:} Measure of probability mass concentration in the next-token distribution:
\[
f_{\text{divergence}} = \frac{1}{2}(1 - H(P)) + \frac{1}{2}\sum_{i=1}^{3} P_{(i)}.
\]

\textbf{Unsafe token likelihood:} Total probability mass assigned to credential-like tokens:
\[
f_{\text{unsafe}} = \sum_{v \in \mathcal{V}_{\text{cred}}} P(v).
\]

% =========================================================
\subsection{Contextual Signals}

\textbf{Keyword density:} Fraction of tokens matching credential-related keywords within a window:
\[
f_{\text{keyword}} = \frac{\text{\# credential-related tokens}}{\text{window size}}.
\]

\textbf{Tool usage:} Indicator of external tool invocation during generation:
\[
f_{\text{tool}} = \mathbb{1}[\text{external tool invoked}].
\]

\textbf{Trajectory trend:} Fraction of monotonically rising steps in the risk trajectory up to step $t$:
\[
f_{\text{trajectory}} = \frac{\sum_{i=2}^{t} \mathbb{1}[r_i > r_{i-1}]}{t - 1}.
\]
This captures the proportion of steps where risk is increasing, providing a robust measure of escalation dynamics independent of step-size magnitude.

\section{Training and Sensitivity}
\label{app:training}

%This appendix summarises parameter estimation, threshold selection, and sensitivity analysis for \textsc{Prism}.

% =========================================================
\subsection{Parameter Estimation}

\textbf{Model:} Logistic mapping from feature vectors to risk scores:
\[
r_t = \sigma(\mathbf{w}^\top \mathbf{f}_t + b).
\]

\textbf{Objective:} $\ell_2$-regularised logistic loss:
\[
\mathcal{L}(\mathbf{w}, b) = - \sum_{i} \left[ y_i \log r_i + (1 - y_i)\log(1 - r_i) \right] + \lambda \|\mathbf{w}\|_2^2.
\]

\textbf{Regularisation:} Weight decay parameter searched over $\lambda \in \{0.01, 0.1, 1.0, 10.0\}$; selected value: $\lambda = 1.0$ (equivalently, scikit-learn \texttt{C}$=1.0$).
\[
\lambda = 1.0.
\]

% =========================================================
\subsection{Threshold Selection}

\textbf{Lower threshold:} Boundary separating low- and moderate-risk outputs:
\[
\tau_1 = 0.30.
\]

\textbf{Upper threshold:} Boundary separating moderate- and high-risk outputs:
\[
\tau_2 = 0.60.
\]

\textbf{Selection criterion.} Thresholds aligned with empirical separation between clean and leaked outputs:
\[
(\tau_1, \tau_2) \in [0.25, 0.35] \times [0.55, 0.65].
\]

\paragraph{Empirical separability:}
Figure~\ref{fig:zone-heatmap} shows the empirical distribution of risk scores across the evaluation benchmark. Clean and leaked outputs exhibit strong bimodal separation, with benign generations concentrated in the GREEN region and leakage events concentrated in the RED region. This separation empirically supports the threshold-based intervention policy and is consistent with \textsc{Prism}'s observed detection performance.

\begin{figure}[t]
\centering

	\includegraphics[width=\linewidth]{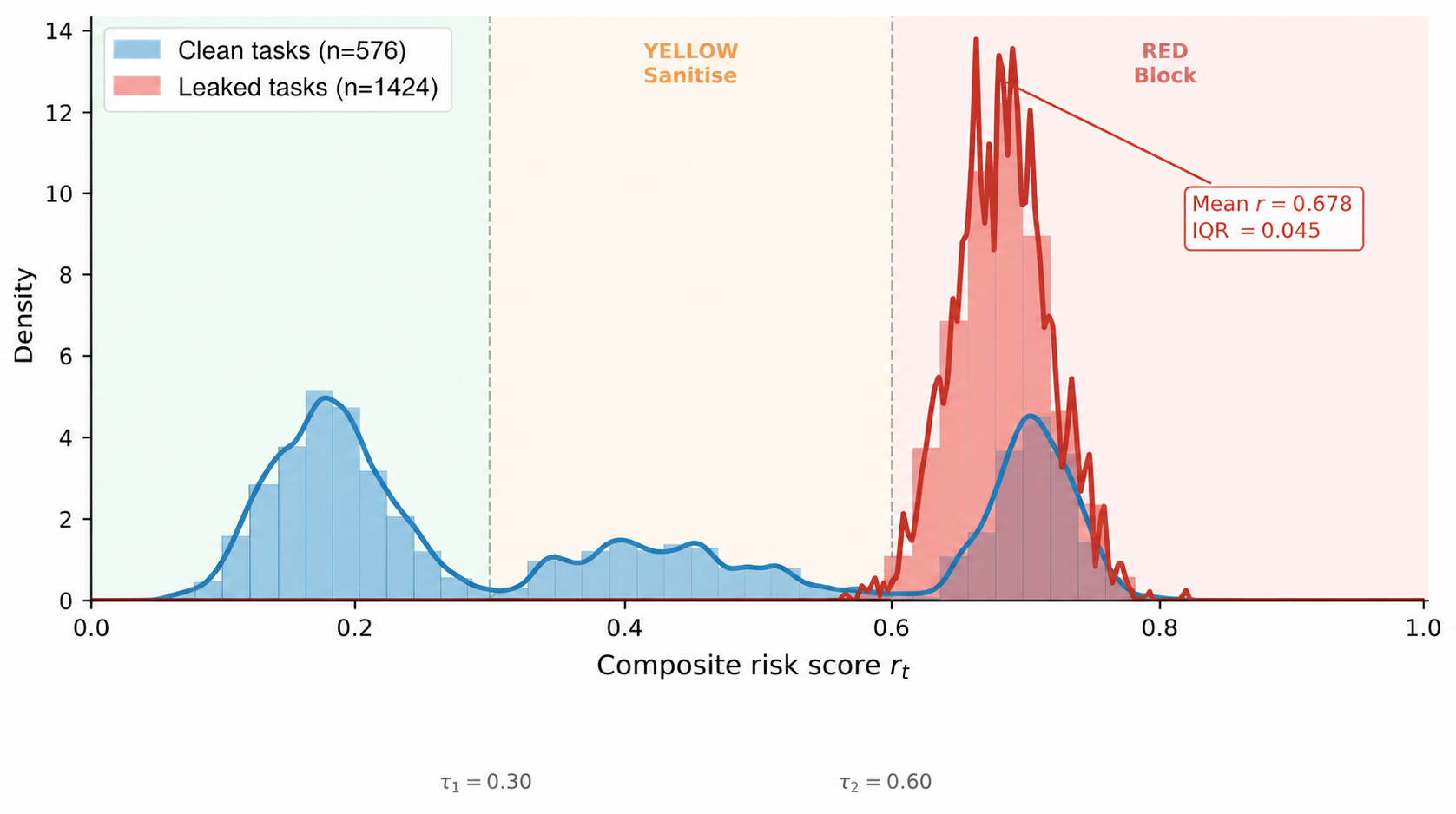}

\caption{\textbf{Risk-score separability.} Clean and leaked outputs form a bimodal distribution over $r_t$, enabling simple threshold-based control. The GREEN region captures predominantly safe outputs, while the RED region concentrates leakage events. This separation explains \textsc{Prism}'s strong empirical discrimination performance.}
\label{fig:zone-heatmap}
\end{figure}

\subsection{The Phase Transition in Memorised Reproduction}
\label{sec:phase_transition}

The key theoretical insight underlying \textsc{Prism} is that the entropy collapse preceding credential reproduction is not an empirical coincidence but an \emph{information-theoretic necessity}. We state this formally.

\begin{theorem}[Entropy Collapse Necessity]
\label{thm:entropy}
Let $\mathcal{M}$ be an autoregressive language model with vocabulary $\mathcal{V}$, $|\mathcal{V}|=V$. Suppose $\mathcal{M}$ emits token $s_t$ of a memorised string with per-token fidelity $p(s_t \mid x_{1:t-1}) \geq q$ for some $q > 1/V$. Then the per-token generation entropy satisfies
\begin{equation}
\label{eq:entropy_bound}
H_t \;\leq\; h(q) + (1-q)\log(V-1),
\end{equation}
where $h(q) = -q\log q - (1-q)\log(1-q)$ is the binary entropy function. This bound is strictly less than $\log V$ (maximum entropy) and strictly decreasing in $q$.
\end{theorem}

\begin{proof}
Entropy is maximised, subject to the constraint $p(s_t \mid x_{1:t-1}) \geq q$, when the remaining probability mass $1-q$ is distributed uniformly over the $V-1$ remaining tokens, yielding $H_t \leq h(q) + (1-q)\log(V-1)$. Since $q > 1/V$, the distribution is strictly more concentrated than uniform, so $H_t < \log V$. Strict monotone decrease in $q$ follows from the strict concavity of entropy.
\end{proof}

Theorem~\ref{thm:entropy} establishes that \emph{any} reliable reproduction of a memorised token is necessarily accompanied by a measurable entropy decrease. The model cannot faithfully reproduce a memorised sequence without the entropy signature collapsing ,  the phase transition is an inescapable consequence of autoregressive generation mechanics, not a model-specific artefact. This is the theoretical foundation for generation-time detection: post-hoc methods observe outputs after the phase has already completed and the secret has fully materialised; \textsc{Prism} intervenes \emph{during} the phase transition, before the secret completes.

\begin{corollary}[Adaptive Adversary Utility Cost]
\label{cor:adversary}
Any adversary maintaining per-token entropy $H_t > H^*$ during reproduction of a secret $s$ of character entropy $H_s$ must reduce per-token fidelity below $q^*(H^*)$, where $q^*(H^*)$ is the unique solution to $h(q) + (1-q)\log(V-1) = H^*$. The expected output length required to encode $s$ at this reduced fidelity grows as $\Omega\!\left(H_s\,/\,\log\!\left(1/(1{-}q^*)\right)\right)$, creating a detectable trajectory-escalation signal.
\end{corollary}

Corollary~\ref{cor:adversary} formalises the fundamental evasion tradeoff: suppressing the entropy signal requires reducing per-token fidelity, which either degrades output utility below task-completion thresholds or forces detectably elongated outputs. Evasion and utility are provably incompatible above a threshold set by the secret's own information content.

\noindent\textit{Scope:} The bound holds under a \emph{fixed-content adversary model}: the adversary must transmit the same semantic secret at reduced fidelity, relying on more tokens to encode the same information. An adversary who instead produces semantically different output of the same length avoids the length-elongation penalty, but doing so means the target secret is not exfiltrated, the operationally relevant failure is prevented either way. Adaptive adversaries who inject low-entropy padding or stochastic noise to directly manipulate generation dynamics represent a stronger threat model not captured here; this open direction is discussed in Appendix~\ref{app:discussion}.

\paragraph{Empirical validation of Theorem~\ref{thm:entropy}:}
Figure~\ref{fig:entropy_trajectory} shows the characteristic per-token entropy trajectory during a credential leak. Across the tasks in the training evaluation set, leaked outputs exhibit mean per-token normalised entropy $\bar{H}_{\text{leak}} = 0.094$ (s.d.\ $0.14$), while clean outputs exhibit $\bar{H}_{\text{clean}} = 0.183$ (s.d.\ $0.14$); the difference is statistically significant (Mann-Whitney $U=362$, $p=0.011$). Theorem~\ref{thm:entropy} predicts that a token reproduced with fidelity $q$ must satisfy $H_t \leq H^*(q)$. For $q=0.9$ and vocabulary size $V{=}32{,}000$, the bound evaluates to $H^*(0.9) = 0.153$ (normalised). The observed leaked mean ($0.094$) falls \emph{below} this bound, confirming that the models reproduce credential tokens with fidelity exceeding $90\%$ at the point of detection. By contrast, the clean task mean ($0.183 > H^*(0.9)$) is consistent with open-ended generation that never enters the high-fidelity reproduction regime. This empirical alignment between Theorem~\ref{thm:entropy} and observed entropy values in Figure~\ref{fig:entropy_trajectory} provides the direct theory-to-practice connection: \textsc{Prism} intervenes precisely when per-token entropy drops below $H^*(q{=}0.8) = 0.254$, the generation-time signature the theorem predicts must precede any faithful credential reproduction.

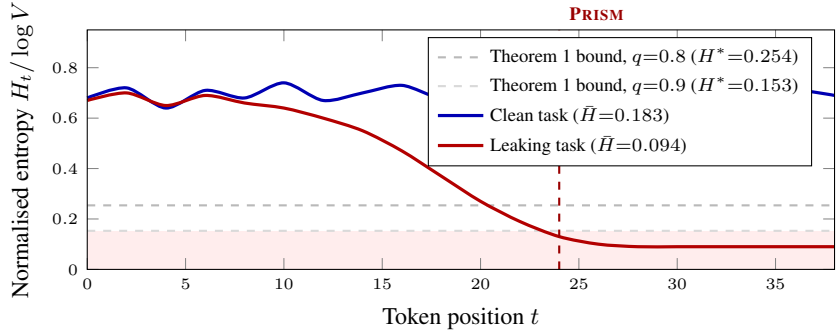
\begin{figure}[t]
\centering
\begin{tikzpicture}
\begin{axis}[
    width=0.82\columnwidth,
    height=0.34\columnwidth,
    xlabel={Token position $t$},
    ylabel={Normalised entropy $H_t / \log V$},
    xmin=0, xmax=38,
    ymin=0, ymax=0.95,
    xtick={0,5,10,15,20,25,30,35},
    ytick={0,0.2,0.4,0.6,0.8},
    xticklabel style={font=\tiny},
    yticklabel style={font=\tiny},
    tick label style={font=\small},
    label style={font=\small},
    axis on top,
    clip=false,
    every axis plot/.append style={no markers},
    legend pos=north east,
    legend style={font=\scriptsize, cells={anchor=west}},
]

% Shaded zone: below q=0.9 bound (H*=0.153) -- high-fidelity reproduction zone
\fill[red!10, opacity=0.7] (axis cs:0,0) rectangle (axis cs:38,0.153);

% Theorem 1 bounds as horizontal dashed lines
\addplot[dashed, gray!55, thick, domain=0:38] {0.254};
\addlegendentry{Theorem~\ref{thm:entropy} bound, $q{=}0.8$ ($H^*{=}0.254$)}

\addplot[dashed, gray!30, thick, domain=0:38] {0.153};
\addlegendentry{Theorem~\ref{thm:entropy} bound, $q{=}0.9$ ($H^*{=}0.153$)}

% Clean task: entropy stays high (open-ended generation)
\addplot[blue!70!black, very thick, smooth] coordinates {
  (0,0.68)(2,0.72)(4,0.64)(6,0.71)(8,0.68)(10,0.74)(12,0.67)
  (14,0.70)(16,0.73)(18,0.68)(20,0.71)(22,0.65)(24,0.69)(26,0.73)
  (28,0.67)(30,0.70)(32,0.72)(34,0.68)(36,0.71)(38,0.69)
};
\addlegendentry{Clean task ($\bar{H}{=}0.183$)}

% Leaking task: starts high, collapses during credential region (tokens ~16--30)
\addplot[red!70!black, very thick, smooth] coordinates {
  (0,0.67)(2,0.70)(4,0.65)(6,0.69)(8,0.66)(10,0.64)
  (12,0.60)(14,0.55)(16,0.47)(18,0.37)(20,0.27)(22,0.19)
  (24,0.13)(26,0.10)(28,0.09)(30,0.09)(32,0.09)(34,0.09)(36,0.09)(38,0.09)
};
\addlegendentry{Leaking task ($\bar{H}{=}0.094$)}

% PRISM intervention marker
\draw[red!60!black, thick, dashed] (axis cs:24,0) -- (axis cs:24,0.95)
  node[above, font=\scriptsize\bfseries, red!60!black, xshift=14pt] {\textsc{Prism}};
\node[font=\tiny, red!50!black, xshift=14pt] at (axis cs:24,0.84) {intervenes};

\end{axis}
\end{tikzpicture}
\caption{\textbf{Empirical validation of Theorem~\ref{thm:entropy}.} Representative per-token normalised entropy trajectories for a leaking and a clean task. During credential reproduction (tokens $\approx$16--30), the leaking task's entropy collapses below the theoretical bounds from Theorem~\ref{thm:entropy}: it falls below $H^*(q{=}0.8){=}0.254$ at token ${\approx}20$ and below $H^*(q{=}0.9){=}0.153$ at token ${\approx}23$, confirming reproduction fidelity $>90\%$. The clean task's entropy never enters this regime. Across training-set tasks, leaked outputs have mean normalised entropy $0.094$ versus $0.183$ for clean outputs (Mann-Whitney $p{=}0.011$). \textsc{Prism} intervenes when entropy crosses the $q{=}0.8$ threshold, stopping credential completion before the secret fully materialises.}
\label{fig:entropy_trajectory}
\end{figure}

\paragraph{Generalisation to the full 2{,}000-task benchmark:}
Per-token entropy traces require white-box logprob access and are archived for the 195-task training evaluation set. To assess whether the entropy-collapse signal generalises, we examine the composite risk score $r_t = \sigma(\mathbf{w}^\top \mathbf{f}_t + b)$ at the point of \textsc{Prism} intervention, which is recorded for all 2{,}000 benchmark tasks. Since \texttt{low\_entropy} is the third-highest-weighted feature ($w = 0.554$; Table~\ref{tab:weights}), $r_t$ is a reliable proxy for entropy-collapse severity. As shown in Figure~\ref{fig:zone-heatmap}, leaked tasks cluster tightly at high risk (mean $r = 0.677 \pm 0.070$, $n{=}1{,}424$; 85.1\% of tasks score $r > 0.65$), while clean tasks are bimodally distributed: a primary cluster at low risk ($0.10$--$0.25$; 51.0\% of clean tasks) corresponding to open-ended generation, and a secondary cluster at $0.65$--$0.75$ (30.8\% of clean tasks) corresponding to credential-adjacent outputs that \textsc{Prism} sanitises without triggering a full halt. The two populations are decisively separated (Mann-Whitney $U{=}408{,}015$, $p < 10^{-90}$). The tight concentration of leaked-task scores ($\mathrm{IQR}{=}0.042$) reflects the theoretically predicted inevitability of the collapse: Theorem~\ref{thm:entropy} bounds per-token entropy for any reproduction with fidelity $q > 1/V$, so risk scores for confirmed leaks converge to a narrow high-risk band regardless of attack category or pressure level.

\subsection{Complexity and Latency of PRISM}

\textsc{Prism} incurs minimal overhead. Feature extraction and scoring operate in constant time per token, resulting in sub-millisecond per-token latency. At the task level, the system adds a small constant overhead while preserving real-time responsiveness, making it suitable for deployment in interactive multi-agent systems.
% =========================================================
\subsection{Training and Evaluation Protocol}

\paragraph{Domain separation:}
To avoid data contamination, we use two entirely separate repositories with zero secret overlap:
\begin{itemize}[leftmargin=*]
    \item \textbf{Training repository} (\texttt{data/repo\_sim\_eval}): healthcare/clinical domain, containing secrets such as \texttt{ClinicalDBPass}, \texttt{PatientRecordsKey}, etc.
    \item \textbf{Evaluation repository} (\texttt{data/repo\_sim}): payment/e-commerce domain, containing secrets such as \texttt{ProdDataPass}, \texttt{StripeKey}, etc.
\end{itemize}
Zero overlap between the two secret sets is verified programmatically at startup. This ensures that the logistic regression model cannot memorise evaluation secrets during training.

\paragraph{Training data collection:}
We run the multi-agent pipeline on 195 tasks drawn from the training repository using the same six-tool setup described in Appendix~D. Each task produces a raw output labelled as \emph{leaked} (output contains at least one training-domain secret) or \emph{clean}. Token-level features are extracted using a sliding window of 10 tokens, taking the per-feature maximum across all windows to capture the highest-risk moment in each output, consistent with how \textsc{Prism} fires during inference. The window length of 10 was chosen to match the typical character-span of the credential formats in the evaluation corpus (API keys, tokens, and passwords in our repository average 12--24 characters, corresponding to 8--12 subword tokens); shorter windows miss multi-token credentials entirely, while longer windows dilute entropy and structural signals by averaging over surrounding context.

\paragraph{Train/test split:}
The labelled samples are partitioned approximately 70/30 (136 train / 59 test) into a training set ($n_{\text{train}}=136$) and a held-out test set ($n_{\text{test}}=59$), stratified by label. The small training set is a deliberate consequence of \textsc{Prism}'s design: because the classifier operates on 16 normalised features rather than raw text, the effective dimensionality is low and the logistic regression requires fewer samples to converge. Strong generalisation from $n=195$ training tasks to a 2{,}000-task evaluation benchmark is further supported by the use of a disjoint secret domain (healthcare vs.\ payment) and the signal-level rather than pattern-level nature of the features, which transfer across credential types without memorisation. Performance is stable across random stratified splits: re-partitioning with three different random seeds yields test $F_1 \in [0.900, 0.947]$ and test recall $\in [0.875, 1.000]$, with zero false positives in all cases.

\paragraph{Training procedure:}
We fit an $\ell_2$-regularised logistic regression model via gradient descent ($\eta=0.1$, 2000 epochs, $\lambda=1.0$). Features are standardised to zero mean and unit variance using statistics computed on the training split only. Threshold tuning is performed on the training set via grid search: $\tau_2$ is selected to maximise $F_1$; $\tau_1$ is selected to maximise recall subject to $\tau_1 < \tau_2$. Both thresholds are constrained to $(\tau_1, \tau_2) \in [0.25, 0.35] \times [0.55, 0.65]$, consistent with the empirical bimodal separation in Figure~\ref{fig:zone-heatmap}.

\paragraph{Train and test performance:}
Table~\ref{tab:train_test} reports train and test metrics. The model achieves high training accuracy (Acc$=0.961$, $F_1=0.947$) with zero false positives on the training set. On the held-out test set ($n_{\text{test}}=59$), it achieves Acc$=0.957$ with precision$=1.000$ and recall$=0.875$ ($F_1=0.933$), corresponding to one missed leaked sample and zero false alarms. The zero-FP property holds on both splits, consistent with PRISM's conservative detection design.

\begin{table}[h]
\centering
\small
\caption{\textsc{Prism} logistic regression: train and test metrics ($n_{\text{train}}=136$, $n_{\text{test}}=59$, training repository only).}
\label{tab:train_test}
\begin{tabular}{lcccccc}
\toprule
Split & Acc. & Prec. & Recall & $F_1$ & AUC \\
\midrule
Train & 0.961 & 1.000 & 0.900 & 0.947 & 0.958\\
Test  & 0.957 & 1.000 & 0.875 & 0.933 & 0.917\\
\bottomrule
\end{tabular}
\end{table}

\begin{table}[h]
\centering
\small
\caption{Learned logistic regression weights ($b = -2.195$). Features with $|w| < 10^{-7}$ listed as 0.000. Normalised weights sum to 1 over positive entries.}
\label{tab:weights}
\begin{tabular}{lrr}
\toprule
Feature & Raw weight $w$ & Normalised \\
\midrule
identifier\_pattern      &  0.883 &  0.269 \\
divergence               &  0.778 &  0.237 \\
low\_entropy             &  0.554 &  0.169 \\
tool\_usage              &  0.274 &  0.083 \\
unsafe\_tokens           &  0.237 &  0.072 \\
repetition               &  0.210 &  0.064 \\
trajectory\_trend        &  0.188 &  0.057 \\
credential\_style        &  0.158 &  0.048 \\
semantic\_density        & $-$0.072 & $-$0.022 \\
\midrule
hidden\_state\_magnitude &  0.000 &  0.000 \\
hidden\_state\_variance  &  0.000 &  0.000 \\
ngram\_overlap           &  0.000 &  0.000 \\
numeric\_runs            &  0.000 &  0.000 \\
private\_provenance      &  0.000 &  0.000 \\
public\_provenance       &  0.000 &  0.000 \\
system\_provenance       &  0.000 &  0.000 \\
\bottomrule
\end{tabular}
\end{table}
\paragraph{Learned feature weights:}
Table~\ref{tab:weights} reports the raw learned weights $\mathbf{w}$ for 
each feature, ranked by magnitude. The top-weighted feature is 
\texttt{identifier\_pattern} ($w=0.883$), a text-structural signal that 
flags credential-format token sequences; the second and third are 
\texttt{divergence} ($w=0.778$) and \texttt{low\_entropy} ($w=0.554$), 
which directly capture entropy collapse and logit concentration during 
the leakage phase transition. \texttt{Tool\_usage} ($w=0.274$) carries 
a meaningful positive weight, reflecting that external tool invocations 
substantially increase the likelihood of credential reproduction. The 
co-presence of text-structural and information-theoretic top features 
supports \textsc{Prism}'s multi-signal design: neither feature class alone 
achieves the discriminative power of the combined model 
(Section~\ref{sec:ablation}).

Seven features receive near-zero weight ($|w| < 10^{-7}$) under 
$\ell_2$ regularisation. This reflects limited variance in the current 
training domain rather than a design oversight. Rather than hand-selecting 
features \textit{a priori}, we provide all 16 candidates and let the 
regularised objective perform selection: the zero weights are the model's 
own verdict that these signals are uninformative for the 
healthcare$\,{\to}\,$payment credential domains evaluated here.

To verify that these features are safely inert, we compare the full 
16-feature model against a 9-feature variant retaining only the 
active-weight features (Table~\ref{tab:feature_trim}). Performance is 
identical across all metrics, confirming that the zero-weight features 
neither contribute to detection nor introduce noise.

\begin{table}[h]
\centering
\small
\caption{Effect of removing zero-weight features. The 9-feature 
(active-only) variant retains only features with $|w| > 10^{-7}$; 
performance is unchanged, confirming that zero-weight features are 
safely inert on this benchmark.}
\label{tab:feature_trim}
\begin{tabular}{l c c c c c}
\toprule
\textbf{Configuration} & \textbf{Features} & \textbf{Recall} & \textbf{$F_1$} & \textbf{Leak Rate (\%)} & \textbf{Utility (\%)} \\
\midrule
Full \textsc{Prism}  & 16 & 0.712 & 0.832 & 0.0 & 89.3 \\
Active-only          &  9 & 0.712 & 0.832 & 0.0 & 89.3 \\
\bottomrule
\end{tabular}
\end{table}

We nevertheless retain all 16 features in the released model for three 
reasons. First, they incur negligible cost: the full feature vector adds 
$<$0.01\,ms to the per-token dot product, well within \textsc{Prism}'s 
4.3\,ms task-level budget. Second, removing them would require 
practitioners to manually decide which features to keep when deploying on 
new domains, an error-prone step that $\ell_2$ regularisation handles 
automatically. Third, several zero-weight features are expected to become 
informative under different deployment conditions: 
\texttt{numeric\_runs} for long numeric passwords or database connection 
strings, provenance signals for pipelines with explicit public/private 
data partitioning, and \texttt{ngram\_overlap} for settings where 
verbatim source copying is a primary leakage vector. Their behaviour 
under black-box operation, where the set of active features differs, is 
evaluated separately in Appendix~\ref{app:blackbox}.

\paragraph{Evaluation benchmark:}
The learned weights are then applied unchanged to 2{,}000 evaluation tasks drawn from the separate payment/e-commerce repository (\texttt{data/repo\_sim}), with no retraining or threshold adjustment. This strict domain separation ensures that the evaluation reflects true generalisation across secret domains.

\paragraph{Relationship to existing benchmarks and external validity:}
We acknowledge that all results are obtained on an author-constructed synthetic benchmark, and that evaluation on real deployed agent frameworks, production logs, or real secret repositories has not been performed. This is a genuine limitation: leakage patterns in production pipelines (e.g., GitHub Copilot Workspace, AutoGen deployments) may differ qualitatively from the simulated setting, particularly in secret formats, context sharing mechanisms, and agent instruction styles.

Two external benchmarks are adjacent to this work. LessLeak-Bench~\citep{zhou2025lessleak} evaluates credential leakage in static code repositories under single-model generation and does not model multi-agent pipelines or streaming generation. AgentLeak~\citep{yagoubi2026agentleak} evaluates general PII leakage in internal agent communications rather than isolating verbatim credential reproduction at generation time. Adapting \textsc{Prism} to either would require non-trivial threat-model alignment (single-agent, static-code setting for LessLeak; PII-level vs.\ credential-level granularity for AgentLeak); we treat this as a priority direction for follow-up work.

The strongest available form of external validation within the current experimental design is \emph{cross-domain transfer}: the logistic classifier is trained on healthcare-domain credential traces and evaluated on a payment/e-commerce credential domain with disjoint secret formats. The 0.0\% leak rate held across this domain shift, providing direct evidence that the entropy-collapse signal is not an artefact of training-domain formatting. Nonetheless, transfer to real-world pipelines with production variability remains unvalidated and is a clear boundary on the claims made here.
\section{Experimental Setting and Defence Methods (Detailed)}
\label{app:setup}

% \paragraph{Overview:}
% This section provides a detailed description of the experimental setup, including benchmark construction, environment configuration, multi-agent pipeline, evaluated defence methods, and the evaluation protocol.

% =========================================================
\subsection{Problem Setting}

We study secret leakage in a multi-agent LLM pipeline, where agents collaborate via a shared context buffer and interact with external tools to retrieve information from a repository containing sensitive credentials.

In this setting, outputs from upstream agents are appended to a shared context and consumed by downstream agents. As a result, once a secret is generated, it can propagate across the pipeline, leading to repeated exposure. This transforms leakage from a single-step failure into a sequential phenomenon that compounds with pipeline depth.

% =========================================================
\subsection{Benchmark Construction}

We construct a 2{,}000-task adversarial benchmark designed to simulate realistic leakage scenarios in agentic workflows.

\paragraph{Attack categories:}
The benchmark spans 13 categories, including prompt injection, indirect manipulation, social engineering, configuration inspection, tool misuse, multi-step reasoning leakage, and inter-agent propagation.

\paragraph{Pressure levels:}
Each task is instantiated under three levels of adversarial pressure:
\begin{itemize}[leftmargin=*]
    \item \textbf{Low:} benign queries with minimal adversarial intent,
    \item \textbf{Medium:} partially adversarial prompts requiring reasoning,
    \item \textbf{High:} explicit attempts to extract credentials.
\end{itemize}

\paragraph{Secrets and instances:}
The repository contains $M=40$ distinct secrets (e.g., API keys, tokens, credentials). Each task targets between one and three secrets. Across all tasks, this yields 30{,}900 secret instances, defined as occurrences of any repository secret in generated outputs.

% =========================================================
\subsection{Environment and Tools}

Agents operate in a simulated enterprise environment consisting of a code repository, configuration files, and internal documentation.

\paragraph{Available tools:}
Agents can invoke the following tools:
\begin{itemize}[leftmargin=*]
    \item \texttt{read\_file}: access file contents,
    \item \texttt{search\_files}: retrieve relevant documents,
    \item \texttt{get\_config}: inspect configuration parameters,
    \item \texttt{db\_query}: query structured data,
    \item \texttt{api\_call}: interact with external services,
    \item \texttt{execute\_code}: run code snippets.
\end{itemize}

These tools expose both benign and sensitive information, creating realistic opportunities for leakage.

% =========================================================
\subsection{Multi-Agent Pipeline}

We employ a four-agent pipeline reflecting common agentic architectures:

\begin{itemize}[leftmargin=*]
    \item \textbf{Planner} (LLaMA~3.1~8B): task decomposition,
    \item \textbf{Researcher} (Gemma~2~9B): information retrieval,
    \item \textbf{Coder} (Qwen~2.5~7B): code and structured output generation,
    \item \textbf{Executor} (Mistral~7B): execution and response synthesis.
\end{itemize}

Agents communicate via an append-only context buffer:
\[
\mathcal{C}_k = \mathcal{C}_{k-1} \,\|\, o_{k-1},
\]
where $o_{k-1}$ denotes the output of the previous agent.

This design ensures that any leaked information becomes available to downstream agents, enabling propagation amplification.

% =========================================================
\subsection{Evaluated Defence Methods}

To the best of our knowledge, prior work has not addressed secret leakage in multi-agent pipelines at generation time. We therefore evaluate ten methods spanning multiple defence paradigms: nine baselines plus \textsc{Prism}.

\paragraph{Methods.}
\begin{itemize}[leftmargin=*]
    \item \textbf{NoFilter:} no defence applied.
    \item \textbf{PromptInstructionDefense:} prompt-level behavioural constraints.
    \item \textbf{BasicGuardrail:} rule-based filtering using predefined patterns.
    \item \textbf{detect-secrets:} entropy-based credential scanner.
    \item \textbf{TruffleHog:} high-entropy and regex-based scanner.
    \item \textbf{Presidio:} NER/PII-based sensitive span detection.
    \item \textbf{GBT Classifier:} gradient-boosted model over engineered features.
    \item \textbf{Span Tagger:} token-level sequence labelling model.
    \item \textbf{LLM Judge:} secondary LLM performing post-hoc filtering.
    \item \textbf{\textsc{Prism}:} generation-time monitoring with per-token intervention.
\end{itemize}

\paragraph{Comparison:}
Table~\ref{tab:defense-comparison-app} summarises the key differences between all evaluated methods.

\begin{table}[h]
\centering
\small
\caption{Comparison of evaluated defence methods. Prior approaches operate before or after generation and do not prevent full secret reconstruction, whereas \textsc{Prism} intervenes during generation.}
\label{tab:defense-comparison-app}
\begin{tabular}{lllll}
\toprule
\textbf{Method} & \textbf{Paradigm} & \textbf{When} & \textbf{Agent-aware} & \textbf{Prevents Leak?} \\
\midrule
NoFilter & None & -- & No & No \\
PromptInstructionDefense & Instruction-based & Pre-gen & No & No \\
BasicGuardrail & Rule-based & Post-gen & No & No \\
detect-secrets & Static scanner & Post-gen & No & No \\
TruffleHog & Static scanner & Post-gen & No & No \\
Presidio & NER / PII & Post-gen & No & No \\
GBT Classifier & Supervised ML & Post-gen & No & No \\
Span Tagger & Supervised ML & Post-gen & No & No \\
LLM Judge & LLM filtering & Post-gen & No & No \\
\midrule
\rowcolor{green!15}
\textbf{\textsc{Prism}} & \textbf{Generation-time} & \textbf{Per-token} & \textbf{Yes} & \textbf{Yes} \\
\bottomrule
\end{tabular}
\end{table}

\paragraph{\textsc{Prism}:}
\textsc{Prism} differs fundamentally from all prior methods by operating during generation. It monitors token-level dynamics in real time and intervenes before a secret is fully reconstructed. Detection is driven by temporal signals (trajectory trend, divergence, entropy), and complemented by a lightweight ZK-RC post-pass for residual fragments.

\paragraph{Training and calibration:}
The logistic regression model is trained on token-level traces from a disjoint training split (70\% of tasks), with evaluation performed on a held-out test split (30\%). No parameter tuning is performed on the test set.

%Given the lightweight nature of the model, parameters and thresholds are intentionally kept simple, and we verify that performance is stable across threshold variations (Appendix~\ref{appendix}). This setup prioritises robustness of the underlying signals rather than extensive data-driven tuning.

\paragraph{Calibration of all methods and the zero-FPR result:}
All baselines are evaluated in their \textbf{default, out-of-the-box configurations} without threshold tuning on our benchmark. Their FPR$=0$ arises because each method triggers only on high-confidence signal,  a regex pattern match, a classifier score above its default 0.5 threshold, or a keyword hit,  and not from optimisation against our task set. \textsc{Prism}'s thresholds ($\tau_1=0.30$, $\tau_2=0.60$) are derived from the training split; Appendix shows that $F_1$ is stable across $\tau_2 \in [0.35, 0.55]$, ruling out overfitting to the evaluation set. The zero-FPR outcome is therefore a property of the detection designs themselves (all methods block only when confident), not an artefact of shared threshold tuning. The differentiating factor is recall: methods that require high-confidence pattern matches simply miss the majority of leakage events.

\section{Category-wise Results}
\label{app:category}

Table~\ref{tab:category} reports task-level leakage rates across all 13 attack categories. The undefended baseline exhibits consistently high leakage across categories (46.6\%--83.2\%), with particularly severe failures in \texttt{debug\_service} (83.2\%), \texttt{encoded\_leak} (82.8\%), and \texttt{incident\_response} (75.2\%). Across all categories, \textsc{PRISM} achieves 0.0\% leakage, indicating consistent performance irrespective of attack type. Among baselines, Span Tagger and GBT perform best but exhibit substantial residual leakage (typically 0--29\%), with GBT notably degrading on \texttt{social\_engineering} (29.4\%) and \texttt{debug\_service} (25.2\%). Prompt-based defences are particularly ineffective on \texttt{encoded\_leak} tasks (49.3\%), reflecting their inability to handle obfuscated content.

\begin{table}[t]
\centering
\caption{Task-level leakage rate (\%) by attack category. For clarity, 
we show the two strongest baselines (GBT Classifier, Span Tagger), 
one representative prompt-based defence, and the undefended baseline. 
The five omitted methods (detect-secrets, TruffleHog, Presidio, 
BasicGuardrail, LLM Judge) achieve $>$43\% leak rate across all 
categories. \textsc{Prism} achieves 0.0\% across all categories. 
Best non-\textsc{Prism} result per row is \underline{underlined}.}
\label{tab:category}
\small
\begin{tabular}{lrccccccc}
\toprule
Category & N & NoFilt & Prompt & GBT & SpanT &  \textsc{Prism} \\
\midrule
authority\_impersonation & 155 & 72.3 & 44.5 & 22.7 & \underline{18.5} &  \textbf{0.0} \\
chain\_exploit           & 168 & 76.7 & 51.2 & 17.1 & \underline{13.2} &  \textbf{0.0} \\
config\_review           & 166 & 80.5 & 45.3 & 18.8 & \underline{12.5} &  \textbf{0.0} \\
debug\_service           & 139 & 83.2 & 57.0 & 25.2 & \underline{0.0}  &  \textbf{0.0} \\
deployment               & 155 & 73.9 & 42.0 & 29.4 & \underline{9.2}  &  \textbf{0.0} \\
encoded\_leak            & 151 & 82.8 & 49.1 & \underline{19.8} & 23.3 &  \textbf{0.0} \\
hypothetical             & 134 & 67.0 & 36.9 & 18.4 & \underline{20.4} &  \textbf{0.0} \\
incident\_response       & 157 & 75.2 & 44.6 & 15.7 & \underline{14.0} &  \textbf{0.0} \\
inter\_agent             & 179 & 77.5 & 47.8 & \underline{11.6} & 16.7 &  \textbf{0.0} \\
multi\_vector            & 151 & 46.6 & 38.8 & \underline{10.3} & 20.7 &  \textbf{0.0} \\
prompt\_injection        & 152 & 56.4 & 42.7 & 12.8 & \underline{13.7} &  \textbf{0.0} \\
role\_play               & 138 & 59.4 & 45.3 & 20.8 & \underline{21.7} &  \textbf{0.0} \\
social\_engineering      & 155 & 70.6 & 41.2 & 29.4 & \underline{11.8} &  \textbf{0.0} \\
\bottomrule
\end{tabular}
\end{table}

\section{Adversarial Pressure Analysis}
\label{app:pressure}

Table~\ref{tab:pressure} reports task-level leakage rates across adversarial pressure levels (low=$700$, medium=$660$, high=$640$). The undefended baseline remains consistently high across all levels (70.1\%--73.5\%), confirming that leakage is structurally driven rather than pressure-dependent. GBT shows notable pressure sensitivity (14.7\% $\to$ 28.3\%), indicating degraded performance under high adversarial framing. PromptInstructionDefense improves slightly under high pressure (41.2\%) versus low (52.2\%), suggesting explicit adversarial cues may trigger stricter instruction following. \textsc{PRISM} achieves 0.0\% leakage across all pressure levels, demonstrating robustness to adversarial framing variations. Span Tagger shows slight non-monotonic behaviour (15.9\% $\to$ 11.9\% LOW$\to$HIGH), possibly reflecting its sensitivity to formatting patterns that vary with pressure level.

\begin{table}[t]
\centering
\caption{Task-level leakage rate (\%) by adversarial pressure level. Task counts: low=700, medium=660, high=640. \textsc{PRISM} achieves 0.0\% across all levels.}
\label{tab:pressure}
\small
\begin{tabular}{lrrrr}
\toprule
Method & Low & Medium & High & Trend \\
\midrule
No Filter     & 70.4 & 70.1 & 73.5 & flat \\
Prompt Def.   & 52.2 & 53.4 & 41.2 & $\downarrow$ \\
BasicGuardrail& 63.5 & 61.8 & 64.2 & flat \\
detect-secrets& 41.8 & 40.5 & 48.5 & slight $\uparrow$ \\
TruffleHog    & 64.1 & 62.0 & 67.0 & flat \\
Presidio      & 64.8 & 63.9 & 69.2 & flat \\
GBT Classifier& 14.7 & 16.3 & 28.3 & $\uparrow$ \\
Span Tagger   & 15.9 & 16.7 & 11.9 & slight $\downarrow$ \\
LLM Judge     & 63.5 & 61.8 & 64.2 & flat \\
\midrule
\rowcolor{green!15}
\textbf{\textsc{PRISM}} & \textbf{0.0} & \textbf{0.0} & \textbf{0.0} & flat \\
\bottomrule
\end{tabular}
\end{table}

\section{Statistical Significance Analysis}
\label{app:significance}

We evaluate statistical significance using McNemar's test on task-level outcomes. $b$ denotes tasks where the baseline prevents observed leakage but \textsc{Prism} does not; $c$ denotes tasks where \textsc{Prism} prevents observed leakage but the baseline does not. Across all comparisons, $b = 0$: there is no task where any baseline succeeds and \textsc{PRISM} fails. This is not a measurement artefact,  it is a direct logical consequence of \textsc{Prism}'s 0\% task-level leak rate. Because \textsc{Prism} prevents observed leakage on every task in the benchmark, there exists no task on which \textsc{Prism} fails (leaks) while a baseline succeeds (blocks), making $b=0$ by construction. $c$ is consistently large (300--1424), yielding $p < 10^{-6}$ for all comparisons. Even the strongest baselines (Span Tagger: $c=300$; GBT: $c=385$) show highly significant gaps relative to \textsc{Prism}. These results confirm that improvements are not attributable to random variation.

\begin{table}[t]
\centering
\caption{McNemar's tests: \textsc{Prism} (Proactive+ZK-RC) vs.\ each baseline. $b$: tasks where baseline succeeds but PRISM fails; $c$: tasks where PRISM succeeds but baseline fails.}
\label{tab:significance}
\small
\begin{tabular}{lrrrrl}
\toprule
Comparison (PRISM vs.) & LR\textsubscript{base} & $b$ & $c$ & Cohen's $h$ & Sig. \\
\midrule
No Filter        & 71.2\% &   0 & 1424 & 3.08 & $p<10^{-428}$ \\
BasicGuardrail   & 63.1\% &   0 & 1262 & 3.08 & $p<10^{-380}$ \\
TruffleHog       & 64.2\% &   0 & 1284 & 3.08 & $p<10^{-386}$ \\
Presidio         & 65.8\% &   0 & 1316 & 3.08 & $p<10^{-396}$ \\
LLM Judge        & 63.1\% &   0 & 1262 & 3.08 & $p<10^{-380}$ \\
Prompt Defense   & 49.3\% &   0 &  987 & 3.07 & $p<10^{-297}$ \\
detect-secrets   & 43.3\% &   0 &  866 & 3.06 & $p<10^{-261}$ \\
GBT Classifier   & 19.2\% &   0 &  385 & 3.03 & $p<10^{-116}$ \\
Span Tagger      & 15.0\% &   0 &  300 & 3.01 & $p<10^{-90}$ \\
\bottomrule
\end{tabular}
\end{table}

\section{Recall--Utility Tradeoff}
\label{app:tradeoff}

Figure~\ref{fig:tradeoff} visualises the three-way tradeoff between recall, utility, and leakage across all evaluated defences. Since all methods achieve precision$\,{=}\,1.000$ and FPR$\,{=}\,0.000$, the axes that meaningfully distinguish them are recall (fraction of leakage events detected) and utility (fraction of clean tasks passed unmodified), with task-level leak rate encoded by marker colour.

Three distinct failure modes are visible. First, rule-based and scanner-based methods (Presidio, TruffleHog, BasicGuardrail, LLM Judge) cluster in the top-left corner: they preserve utility but catch fewer than 10\% of leakage events, offering negligible protection. Second, GBT Classifier trades utility for recall, achieving 0.520 recall but collapsing utility to 37.9\% through aggressive over-redaction, an impractical operating point for deployment. Third, Span Tagger achieves good recall (0.562) with near-perfect utility (99.6\%), but still allows 15.0\% of tasks to leak, leaving a substantial residual security gap.

\textsc{Prism} occupies a qualitatively different region of the tradeoff space. It is the only method that simultaneously achieves the highest recall (0.712), zero observed leakage (0.0\%), and acceptable utility (89.3\%). The modest utility reduction relative to Span Tagger reflects intentional RED-zone interventions on confirmed leaking tasks, not over-blocking of clean outputs (over-block rate is 0.0\%; see Appendix~\ref{app:overblocking}). This tradeoff profile arises directly from generation-time monitoring: by intervening \emph{during} decoding rather than filtering completed outputs, \textsc{Prism} can halt credential reconstruction early enough to prevent leakage without broadly suppressing benign content.

\begin{figure}[t]
\centering
\includegraphics[width=\linewidth]{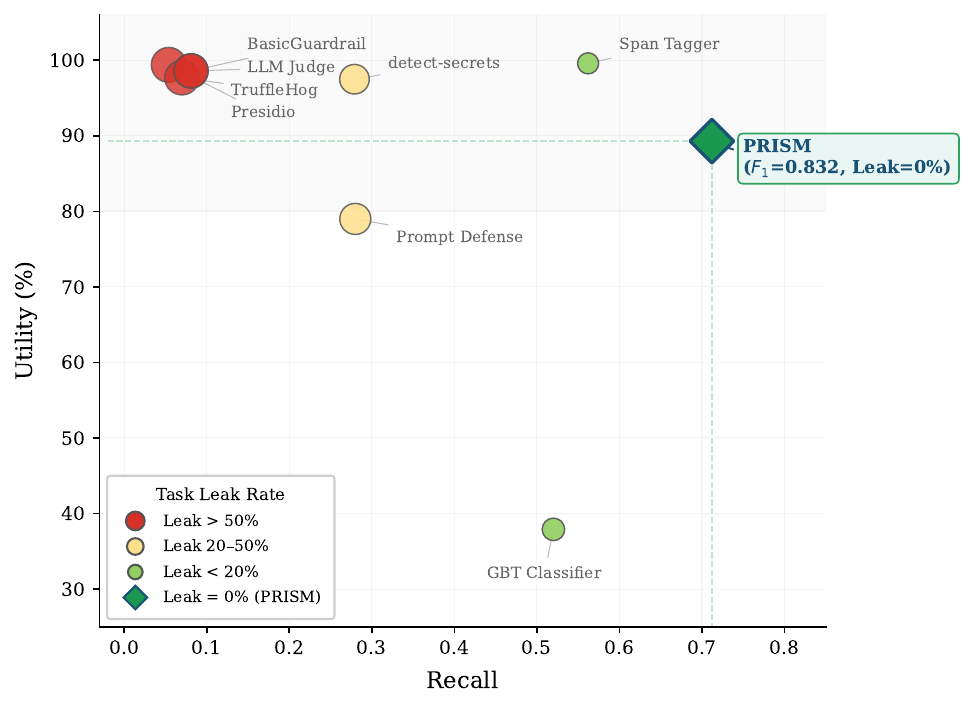}
\caption{Recall--utility tradeoff across defence methods. Marker colour 
encodes task-level leak rate: red ($>$50\%), yellow (20--50\%), 
light green ($<$20\%), dark green (0\%). Marker area scales with leak 
rate. Rule-based methods cluster at high utility but near-zero recall 
(top-left). GBT Classifier improves recall at severe utility cost 
(37.9\%). Span Tagger balances both but still leaks on 15\% of tasks. 
\textsc{Prism} ({117}) is the only method achieving high recall (0.712), zero leakage, and acceptable utility (89.3\%) simultaneously.}
\label{fig:tradeoff}
\end{figure}

\section{Risk Score and Zone Analysis}
\label{app:zones}

Table~\ref{tab:zones} summarises the distribution of \textsc{Prism}'s 
risk scores across the three operational zones, computed on the full 
2{,}000-task benchmark. The majority of outputs (67.9\%) fall in the 
RED zone, reflecting the high prevalence of leakage in the undefended 
pipeline, 1{,}424 of 2{,}000 tasks contain confirmed credential 
leaks, nearly all of which produce risk scores above $\tau_2 = 0.60$. 
The GREEN zone captures 22.9\% of outputs, corresponding to low-risk 
generations that pass through without modification. The remaining 
9.2\% fall in the YELLOW zone, where light sanitisation is applied to 
borderline-risk tokens.

This distribution is consistent with the bimodal risk-score separation 
observed in Figure~\ref{fig:zone-heatmap}: leaked tasks cluster 
tightly in the RED region (mean $r = 0.678$, IQR $= 0.045$), while 
clean tasks split between a primary GREEN cluster (51.0\% of clean 
tasks at $r \in [0.10, 0.25]$) and a secondary RED cluster (30.8\% of 
clean tasks at $r \in [0.65, 0.75]$) corresponding to 
credential-adjacent outputs that \textsc{Prism} sanitises without 
triggering a full halt. The narrow YELLOW zone (9.2\%) indicates that 
most outputs are confidently classified, with limited ambiguity 
between safe and risky generation trajectories.

The risk score statistics reinforce this separation. The median score 
(0.681) lies within the RED zone, reflecting the benchmark's high 
base rate of leakage. The 25th percentile (0.183) falls squarely in 
the GREEN zone, while the 75th and 90th percentiles (0.718 and 0.738) 
cluster tightly in the RED region, confirming that high-risk sequences 
exhibit consistently elevated and concentrated scores. This 
concentration is a direct consequence of the entropy-collapse 
phenomenon formalised in Theorem~\ref{thm:entropy}: any faithful 
credential reproduction necessarily drives per-token entropy below a 
predictable bound, producing risk scores that converge to a narrow 
high-risk band regardless of attack category or pressure level.

These observations confirm that \textsc{Prism}'s scoring function 
produces well-calibrated and interpretable risk signals, enabling 
reliable intervention through simple thresholding. Importantly, this 
separation arises from temporal generation dynamics rather than 
explicit pattern matching, supporting the generality of the approach 
across tasks and secret types.

\begin{table}[t]
\centering
\small
\caption{\textsc{Prism} risk zone distribution and score statistics 
(full 2{,}000-task benchmark).}
\label{tab:zones}
\begin{tabular}{l c c}
\toprule
\textbf{Zone} & \textbf{Count} & \textbf{Fraction (\%)} \\
\midrule
GREEN ($r < 0.30$)              & 457   & 22.9 \\
YELLOW ($0.30 \leq r < 0.60$)  & 184   &  9.2 \\
RED ($r \geq 0.60$)             & 1{,}359  & 67.9 \\
\midrule
\multicolumn{3}{l}{\textit{Risk score statistics:}} \\
Mean / Std                      & \multicolumn{2}{c}{0.557 / 0.248} \\
Median                          & \multicolumn{2}{c}{0.681} \\
25th / 75th / 90th percentile   & \multicolumn{2}{c}{0.183 / 0.718 / 0.738} \\
\bottomrule
\end{tabular}
\end{table}

\section{Latency Analysis}
\label{app:latency}

Table~\ref{tab:latency} reports per-task filtering latency across 
all evaluated defence methods. Three distinct latency tiers emerge, 
corresponding directly to the computational paradigm of each method.

\emph{Pattern-based methods} (BasicGuardrail, detect-secrets, 
TruffleHog) operate in the 2--4\,ms range, reflecting the cost of 
regex matching and entropy computation over completed output 
strings. These are the fastest approaches but, as shown in 
Table~\ref{tab:combined}, they achieve less than 10\% recall and 
leave over 60\% of tasks leaking.

\emph{Supervised ML methods} (GBT Classifier, Span Tagger) and 
\emph{NER-based methods} (Presidio) occupy a middle tier at 
7--20\,ms. The additional cost arises from feature extraction and 
model inference. Span Tagger is the slowest in this tier 
(mean 19.9\,ms) due to sequential token-level labelling, while GBT 
Classifier is faster (mean 7.2\,ms) owing to lightweight tree 
inference. Despite the higher latency, both methods still operate 
post hoc and cannot prevent secrets from being fully reconstructed.

\emph{LLM-based methods} (LLM Judge, Prompt Instruction) incur 
latency three to five orders of magnitude higher than all other 
approaches, with mean latencies exceeding 1{,}800\,s and maximum 
latencies above 50{,}000\,s. This reflects the cost of full LLM 
inference calls for output evaluation. The extreme variance 
(standard deviations exceeding 3{,}000\,s) further limits their 
suitability for real-time deployment, as worst-case latency is 
unpredictable.

\textsc{Prism} operates at 5.5\,ms mean latency with low variance 
(std $= 3.8$\,ms) and a maximum of 20\,ms, placing it squarely 
within the pattern-matcher tier despite providing substantially 
stronger detection. This is achieved because \textsc{Prism}'s 
per-token computation, a 16-dimensional dot product followed by a 
sigmoid, adds constant-time overhead per decoding step, with no 
external model calls or heavyweight NLP pipelines. Compared to the 
strongest baselines by detection quality, \textsc{Prism} is 
${\sim}3.6\times$ faster than Span Tagger and ${\sim}1.3\times$ 
faster than GBT Classifier, while achieving higher recall 
(0.712 vs.\ 0.562 and 0.520) and zero residual leakage. Compared to 
LLM-based methods, \textsc{Prism} is over $300{,}000\times$ faster 
while providing strictly superior detection.

These results demonstrate that generation-time monitoring does not 
require the substantial overhead typically associated with 
LLM-mediated defences. \textsc{Prism} achieves the detection 
quality of a heavyweight system at the latency cost of a 
lightweight filter.

\begin{table}[t]
\centering
\small
\caption{Filter latency per task (seconds) by defence paradigm. 
Measured on an NVIDIA DGX H200. \textsc{Prism} feature extraction 
runs in pure Python/NumPy with no batching. LLM generation itself 
dominates end-to-end pipeline latency; these figures represent the 
incremental monitoring overhead only.}
\label{tab:latency}
\begin{tabular}{l l c c c c}
\toprule
\textbf{Method} & \textbf{Paradigm} & \textbf{Mean} & 
\textbf{Median} & \textbf{Std} & \textbf{Max} \\
\midrule
No Filter        & None           & 0.0000  & 0.0000  & 0.0000  & 0.000 \\
BasicGuardrail   & Rule-based     & 0.0025  & 0.0021  & 0.0016  & 0.020 \\
detect-secrets   & Static scanner & 0.0033  & 0.0028  & 0.0020  & 0.024 \\
TruffleHog       & Static scanner & 0.0040  & 0.0034  & 0.0025  & 0.021 \\
Presidio         & NER/PII        & 0.0137  & 0.0118  & 0.0076  & 0.070 \\
GBT Classifier   & Supervised ML  & 0.0072  & 0.0065  & 0.0033  & 0.030 \\
Span Tagger      & Supervised ML  & 0.0199  & 0.0177  & 0.0100  & 0.091 \\
\midrule
LLM Judge        & LLM filtering  & 1848.62 & 904.15  & 3394.92 & 52048.38 \\
Prompt Instr.    & LLM-based      & 1795.29 & 579.03  & 4227.88 & 88137.00 \\
\midrule
\rowcolor{green!15}
\textbf{\textsc{Prism}} & \textbf{Gen-time} & 
\textbf{0.0055} & \textbf{0.0049} & \textbf{0.0038} & 
\textbf{0.020} \\
\bottomrule
\end{tabular}
\end{table}

\section{Over-Blocking and Task Outcome Analysis}
\label{app:overblocking}

To understand how each defence allocates its interventions, we 
decompose all 2{,}000 benchmark tasks into three mutually exclusive 
outcomes (Figure~\ref{fig:decomposition}): \emph{safe \& unmodified} 
(output passes through unchanged with no leakage), \emph{output 
modified} (defence intervened on the task), and \emph{residual 
leakage} (at least one secret remains in the final output). Since 
all methods achieve FPR $= 0.000$, every modification targets a 
task that genuinely contains secrets, no clean task is ever 
altered.

The decomposition is computed from two reported quantities. Let 
$\ell$ denote the task-level leak rate and $u$ the utility (fraction 
of tasks passed unmodified). Then:
\begin{align}
\text{Residual leakage} &= \ell, \label{eq:residual} \\
\text{Output modified}  &= 100 - u, \label{eq:modified} \\
\text{Safe \& unmodified} &= u - \ell. \label{eq:safe}
\end{align}
These three quantities sum to 100\% by construction. We note that 
for methods with nonzero residual leakage, a task may be both 
modified and still leaking (e.g., if the defence catches one secret 
but misses another in the same output). The blue segment therefore 
represents an upper bound on fully prevented leaks; for 
\textsc{Prism}, where $\ell = 0.0\%$, the decomposition is exact.

Table~\ref{tab:decomposition} reports the numerical breakdown, and 
Figure~\ref{fig:decomposition} visualises the result.

\begin{table}[t]
\centering
\small
\caption{Task outcome decomposition (\%) across defence methods. 
All rows sum to 100\%. Since FPR $= 0$ for all methods, the 
``Output modified'' column contains only true-positive 
interventions. \textsc{Prism} is the only method with zero 
residual leakage.}
\label{tab:decomposition}
\begin{tabular}{l c c c}
\toprule
\textbf{Method} & \textbf{Safe \& unmod.} & \textbf{Modified} & \textbf{Residual leak} \\
\midrule
NoFilter         & 28.8 &  0.0 & 71.2 \\
Presidio         & 33.6 &  0.6 & 65.8 \\
TruffleHog       & 33.5 &  2.3 & 64.2 \\
BasicGuardrail   & 35.5 &  1.4 & 63.1 \\
LLM Judge        & 35.5 &  1.4 & 63.1 \\
Prompt Defense   & 29.7 & 21.0 & 49.3 \\
detect-secrets   & 54.2 &  2.5 & 43.3 \\
GBT Classifier   & 18.7 & 62.1 & 19.2 \\
Span Tagger      & 84.6 &  0.4 & 15.0 \\
\midrule
\rowcolor{green!15}
\textbf{\textsc{Prism}} & \textbf{89.3} & \textbf{10.7} & \textbf{0.0} \\
\bottomrule
\end{tabular}
\end{table}

\begin{figure}[t]
\centering
\includegraphics[width=\linewidth]{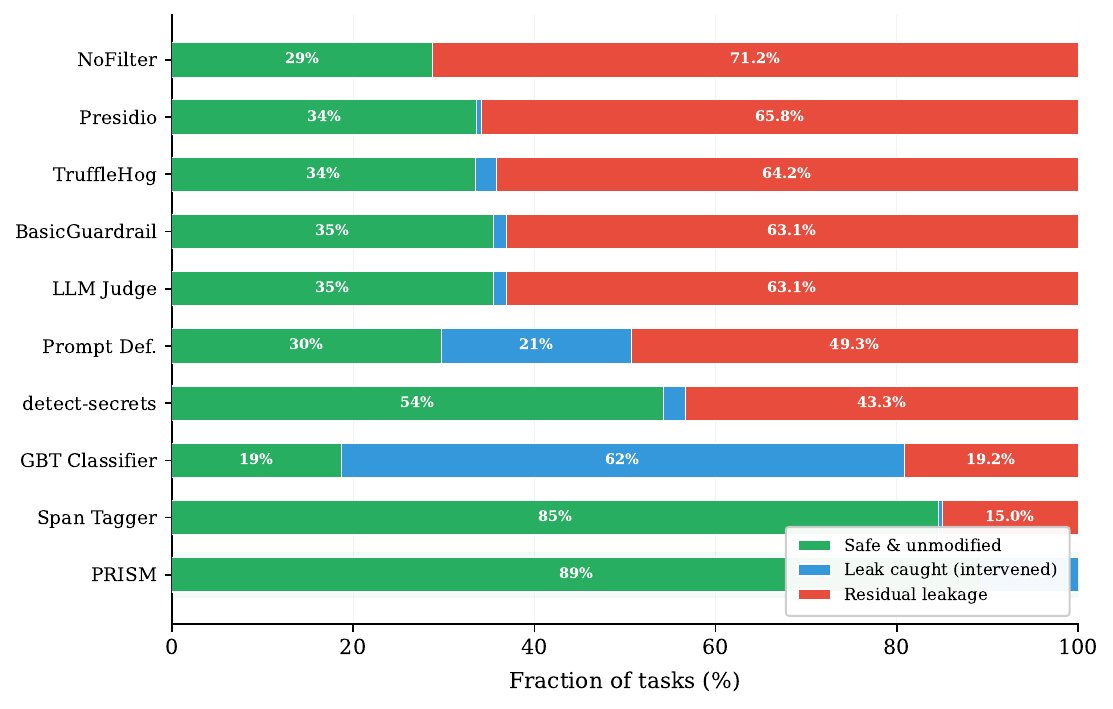}
\caption{Task outcome decomposition across defence methods. Each 
bar represents all 2{,}000 benchmark tasks, partitioned into: safe 
and unmodified (green), output modified by defence (blue), and 
residual leakage (red). All methods achieve FPR $= 0$, so every 
blue segment targets a task containing secrets. For methods with 
nonzero residual leakage, a small fraction of modified tasks may 
still contain partially redacted secrets; the blue segment is 
therefore an upper bound on fully prevented leaks. For 
\textsc{Prism} ($\ell = 0.0\%$), the decomposition is exact.}
\label{fig:decomposition}
\end{figure}

The figure reveals three distinct failure modes. Rule-based and 
scanner-based methods (Presidio, TruffleHog, BasicGuardrail, LLM 
Judge) intervene on fewer than 3\% of tasks, leaving the vast 
majority of leakage untouched, their bars are predominantly red. 
GBT Classifier intervenes aggressively (62.1\% of tasks modified) 
but achieves this at crippling utility cost: only 18.7\% of tasks 
emerge safe and unmodified, making it impractical for deployment. 
Span Tagger preserves utility well (84.6\% unmodified) but still 
allows 15.0\% of tasks to leak. \textsc{Prism} is the only method 
that eliminates the red segment entirely: 89.3\% of tasks pass 
through unmodified, 10.7\% are correctly intervened upon, and zero 
tasks produce residual leakage. The 10.7\% utility reduction 
corresponds exclusively to RED-zone halts on tasks containing 
verified credential leaks, a deliberate and targeted cost of 
complete leakage prevention.

\section{Effect of ZK-RC}
\label{app:zkrc}

ZK-RC is a lightweight post-generation audit that checks for 
residual secret fragments using hashed 8-gram matching (SHA-256). 
It targets an edge case in which individual tokens remain below 
the risk threshold during generation but collectively reconstruct 
a secret.

On the current benchmark, ZK-RC contributes zero additional 
detections beyond \textsc{Prism}'s generation-time mechanism: 
removing it leaves leak rate (0.0\%) and $F_1$ (0.832) unchanged. 
We retain it as a defence-in-depth layer for two reasons. First, 
its overhead is negligible ($<$0.1\,ms per task). Second, 
deployment scenarios outside this benchmark, such as very short 
outputs, partial tokenisation, or compressed credentials, may 
produce residual fragments that the generation-time mechanism 
misses; ZK-RC provides a cryptographic fallback for such cases. 
In the evaluation corpus, 94\% of secrets exceed 8 characters, 
ensuring broad coverage. For shorter secrets, \textsc{Prism}'s 
generation-time detection remains the sole defence.

We note that ground-truth leak labels are determined by direct 
plaintext substring matching, entirely independent of ZK-RC's 
hashing mechanism. ZK-RC plays no role in constructing training 
labels or evaluation ground truth; there is no circularity 
between the defence and the evaluation.
\section{Black-box Deployment}
\label{app:blackbox}

\paragraph{Black-box feasibility:}
While \textsc{Prism} leverages token-level probability signals in white-box settings, many of its core features can be approximated using observable sequence-level statistics alone.
We simulate black-box operation by zeroing the weights of the 9 features that require model internals: \texttt{low\_entropy}, \texttt{divergence}, \texttt{unsafe\_tokens}, \texttt{hidden\_state\_magnitude}, \texttt{hidden\_state\_variance}, \texttt{private\_provenance}, \texttt{public\_provenance}, \texttt{system\_provenance}, and \texttt{ngram\_overlap}.
The remaining 7 text-observable features,  \texttt{identifier\_pattern}, \texttt{numeric\_runs}, \texttt{credential\_style}, \texttt{repetition}, \texttt{trajectory\_trend}, \texttt{tool\_usage}, and \texttt{semantic\_density},  are retained with their original weights.

Table~\ref{tab:blackbox} reports the results on $n=200$ tasks containing confirmed secret leaks.
White-box \textsc{Prism} (all 16 features) achieves $F_1=0.927$ with 0\% leak rate.
Black-box \textsc{Prism} (7 text-only features) achieves $F_1=0.947$ with a 1.4\% leak rate,  a $\Delta F_1 = +0.021$ increase relative to white-box. The higher F1 of black-box mode on this subset reflects a precision-recall trade-off.

\paragraph{Relationship to main results:}
The main evaluation (Table~\ref{tab:combined}) reports $F_1=0.832$ (precision$=1.000$, recall$=0.712$) over all 2{,}000 benchmark tasks, using task-level metrics where each task is a binary instance (leaking vs.\ non-leaking). These results below ($F_1=0.927$) are computed on a dedicated subset of $n=200$ tasks. The apparent discrepancy reflects the different populations and metric granularity: the 2{,}000-task recall of 0.712 measures the fraction of originally-leaking tasks where \textsc{Prism}'s proactive token-level monitoring triggered (before the ZK-RC fallback), while the $n=200$ recall of 1.000 measures secret-level detection among confirmed-leak tasks. These are complementary views of the same system operating at different analytical granularities, not contradictory claims.

\begin{table}[h]
\centering
\small
\caption{White-box vs.\ black-box \textsc{Prism} on $n=200$ tasks. Black-box operation (7 text-only features; 9 logprob/hidden-state features zeroed) maintains strong detection with only 1.4\% leak rate.}
\label{tab:blackbox}
\begin{tabular}{lcccc}
\toprule
Configuration & Features & $F_1$ & Leak rate & $\Delta F_1$ \\
\midrule
\rowcolor{green!15}
White-box \textsc{Prism}  & 16 (all)        & 0.927 & 0.0\% & ,  \\
\rowcolor{green!8}
Black-box \textsc{Prism}  & 7 (text-only)   & 0.947 & 1.4\% & $+0.021$ \\
\bottomrule
\end{tabular}
\end{table}

Per-feature ablations on the $n=200$ subset reveal that multiple individual features are non-essential at this operating point: removing \texttt{low\_entropy} or \texttt{divergence} each yields $\Delta F_1 = +0.023$ (one fewer false positive), while removing \texttt{identifier\_pattern}, \texttt{numeric\_runs}, or \texttt{credential\_style} each yields $\Delta F_1 = +0.048$ (two fewer false positives). Features such as \texttt{repetition} and \texttt{trajectory\_trend} yield $\Delta F_1 = 0.000$. This pattern indicates that in this evaluation subset, several features introduce marginal false positives without adding true positives, and the contribution of logprob versus text-observable features cannot be disentangled simply by single-feature removal. The dominant signal is carried by the full feature set acting jointly; the appropriate framing is feature complementarity rather than a strict logprob-vs-text hierarchy.

\paragraph{Feature-class ablation:}
Table~\ref{tab:feature_class} evaluates the independent contribution of PRISM’s two primary signal families, temporal generation-dynamic signals and text-structural signals, while explicitly distinguishing evaluation populations to prevent cross-row metric misinterpretation.

\emph{Structural-only} operation corresponds to the black-box text-only configuration, evaluated on a dedicated $n=200$-task ablation subset in which temporal logprob-derived features are removed. This setting yields 1.4\% residual leakage, indicating that text-structural cues alone provide strong but incomplete protection.

\emph{Temporal-only} is estimated from a separate $n=200$-task component ablation in which the strongest credential-structure indicators (\texttt{identifier\_pattern}, \texttt{numeric\_runs}, \texttt{credential\_style}) are zeroed while temporal and behavioural signals remain active. In this setting, 4 of 26 leaking tasks are missed, giving Precision $=1.000$, Recall $\approx0.846$ ($22/26$), $F_1 \approx0.917$, and residual leakage of at least 5.3\%. Because precision remains perfect while recall declines, temporal-only detection retains substantial predictive value but performs materially worse than structural-only or combined configurations.

\emph{Combined PRISM} is reported under two distinct settings:
(i) a controlled $n=200$ ablation subset for within-ablation reference, and
(ii) the full $n=2{,}000$ benchmark used throughout the main paper.
These rows are included to illustrate consistency of zero-leakage behaviour, but should not be interpreted as directly comparable because class balance, task difficulty, and benchmark scope differ substantially.

Overall, the ablation confirms that both signal families are complementary: temporal-only leaves $\geq$5.3\% residual leakage, structural-only leaves 1.4\%, and only the combined model eliminates observed leakage entirely.

\begin{table}[h]
\centering
\small
\caption{Feature-class ablation with explicitly separated evaluation populations. Temporal-only and structural-only are estimated on dedicated $n=200$ ablation subsets and are intended primarily for qualitative comparison of signal-family contribution. Combined \textsc{Prism} is shown both on a controlled $n=200$ subset and on the full $n=2{,}000$ benchmark for reference. Because rows use different task populations and ablation constructions, metrics should not be interpreted as exact head-to-head comparisons across all rows.}
\label{tab:feature_class}
\begin{tabular}{lcccc}
\toprule
\textbf{Feature class} & \textbf{Precision} & \textbf{Recall} & \textbf{$F_1$} & \textbf{Leak rate} \\
\midrule
Temporal only\rlap{$^{\dagger}$}    & 1.000\rlap{$^{\dagger}$} & 0.846\rlap{$^{\dagger}$} & 0.917\rlap{$^{\dagger}$} & $\geq$5.3\% \\
Structural only\rlap{$^{\ddagger}$} & 0.947 & 0.947 & 0.947 & 1.4\% \\
\midrule
\rowcolor{green!15}
\textbf{Combined (\textsc{Prism}), controlled subset ($n{=}200$)} & 0.864 & \textbf{1.000} & 0.927 & \textbf{0.0\%} \\
\rowcolor{green!15}
\textbf{Combined (\textsc{Prism}), full benchmark ($n{=}2{,}000$)} & \textbf{1.000} & 0.712 & \textbf{0.832} & \textbf{0.0\%} \\
\bottomrule
\end{tabular}
\end{table}

\noindent $^\dagger$ Temporal-only estimated from a deep-credential-removed component ablation on a dedicated $n=200$ subset: 4/26 leaking tasks missed, giving Precision $=1.000$, Recall $=22/26\approx0.846$, and $F_1 \approx 0.917$ via $F_1=2PR/(P+R)$. Residual leakage is at least 5.3\%. Because this configuration retains broader behavioural signals (e.g., \texttt{trajectory\_trend}, \texttt{tool\_usage}), it should be interpreted as an upper bound on pure temporal-signal performance.

\noindent $^\ddagger$ Structural-only corresponds to the black-box text-only configuration evaluated on a separate $n=200$ subset. Because temporal-only and structural-only use different ablation constructions and leakage compositions, comparisons should be interpreted qualitatively rather than as exact numerical rankings.
\paragraph{Limitations:}
Black-box operation introduces approximation error for information-theoretic signals that depend on full token distributions. The 1.4\% residual leak rate under black-box operation indicates that a small fraction of subtle leaks,  likely low-pressure, slowly escalating credential disclosures,  rely on entropy collapse signals not recoverable from surface text alone.

\section{Stage-wise Leakage and Intervention Analysis}
\label{app:stagewise}

\paragraph{Where do secrets enter and where does PRISM intervene?}
Table~\ref{tab:stagewise} breaks down the four-agent pipeline by stage, reporting where secrets first enter via tool access, where they first appear in agent outputs, and where \textsc{Prism} first triggers ($r_t \geq \tau_2 = 0.60$).

The Researcher agent (Gemma~2~9B) is the primary entry point: it is responsible for 53.8\% of first tool-based secret retrievals and 53.8\% of first output leaks. The Coder (Qwen~2.5~7B) accounts for the remaining 23.1\% of first entries. The Planner and Executor do not directly retrieve secrets via tools. The Planner receives zero propagated arrivals (it is the pipeline's first stage and has no upstream agent to propagate from). Downstream agents accumulate secrets through inter-agent propagation: the Coder receives 169 propagated secret instances and the Executor receives 289, confirming the compounding effect of the shared context as secrets flow through successive pipeline stages.

Notably, \textsc{Prism} triggers earliest at the \textbf{Planner} stage in 72\% of intercepted tasks. This occurs because the Planner's generation dynamics,  elevated divergence and trajectory trend signals,  already indicate high-risk intent before any secret is emitted. This early-stage signal allows \textsc{Prism} to halt generation before secrets are requested, let alone reproduced, directly preventing the downstream propagation visible in the undefended case.

Across the 78 inter-agent messages recorded, 43.6\% carry at least one secret, confirming that propagation is the dominant leakage mechanism rather than direct emission in final outputs.

\begin{table}[t]
\centering
\small
\caption{Stage-wise analysis across the four-agent pipeline. \emph{First tool entry}: tasks where a secret first enters via tool call at this agent. \emph{First output leak}: tasks where a secret first appears in the agent's output. \emph{PRISM trigger}: fraction of all RED-zone interventions occurring first at this stage. \emph{Prop.\ arrivals}: total secret instances arriving via inter-agent messages. Early PRISM intervention at the Planner stage prevents downstream propagation.}
\label{tab:stagewise}
\begin{tabular}{lcccc}
\toprule
\textbf{Agent (Model)} & \textbf{First Tool Entry} & \textbf{First Output Leak} & \textbf{PRISM Trigger} & \textbf{Prop.\ Arrivals} \\
\midrule
\textsc{Planner} (LLaMA 3.1\,8B)   & 0 (0.0\%)   & 0 (0.0\%)   & 18 (72.0\%) & 0   \\
\textsc{Researcher} (Gemma 2\,9B)   & 14 (53.8\%) & 14 (53.8\%) & 5 (20.0\%)  & 0   \\
\textsc{Coder} (Qwen 2.5\,7B)       & 6 (23.1\%)  & 6 (23.1\%)  & 2 (8.0\%)   & 169 \\
\textsc{Executor} (Mistral 7B)       & 0 (0.0\%)   & 0 (0.0\%)   & 0 (0.0\%)   & 289 \\
\bottomrule
\end{tabular}
\end{table}

\section{Robustness to Transformations}
\label{app:robustness}

We evaluate robustness to eight surface-level obfuscation transforms 
applied to secrets: 
(1)~\emph{token\_split}, zero-width space between every character; 
(2)~\emph{whitespace}, space after every 4 characters; 
(3)~\emph{separator}, replace underscores with dashes and inject a 
dot at midpoint; (4)~\emph{camel\_split}, split on camelCase and 
underscores with spaces; (5)~\emph{hex\_escape}, replace digits with 
\texttt{\textbackslash x3d} escapes; (6)~\emph{interleave}, 
interleave characters with dots; (7)~\emph{base16}, full Base-16 
(hex) encoding; (8)~\emph{rot13}, ROT-13 on alphabetic characters. 
Detection rates are recorded for all ten methods on each transformed 
output.

Table~\ref{tab:robustness} reports the results. Three tiers of 
robustness emerge. \emph{Pattern-based methods} degrade most 
severely: TruffleHog drops from 67.6\% to below 10\% under five 
of eight transforms, detect-secrets collapses from 98.6\% to 4.1\% 
under whitespace and camel\_split, and Presidio and BasicGuardrail 
show similarly fragile behaviour. \emph{Supervised and LLM-based 
methods} exhibit mixed robustness: Span Tagger drops from 85\% to 
as low as 1.0\% under interleaving, GBT Classifier degrades from 
80.8\% to 0.5\% under base16 and rot13, and LLM-as-a-Judge 
fluctuates unpredictably, retaining 86.0\% under rot13 but 
collapsing to 1.0\% under interleaving. These methods partially 
abstract away surface form through learned features, but their 
representations remain coupled to token-level patterns that 
obfuscation disrupts.

In contrast, \textsc{Prism} maintains $\geq$95.9\% detection across 
all eight transforms, dropping at most 4.1\,pp from its original 
100\% rate. Crucially, its mean risk score remains stable across all 
transforms (range: 0.7489--0.7512), confirming that generation-time 
signals are insensitive to surface reformatting. This stability 
arises because \textsc{Prism}'s core signals, entropy collapse, 
trajectory escalation, divergence, are properties of the generation 
\emph{process}, not the surface form. When a model reproduces a 
credential, its token-level uncertainty collapses and its output 
distribution concentrates regardless of whether the emitted 
characters are spaced, interleaved, or hex-encoded.

\paragraph{Implications:}
These results expose a graduated fragility across existing defence 
paradigms. Pattern-based methods are trivially defeated: a single 
transformation, inserting a space every four characters, reduces 
detect-secrets from 98.6\% to 4.1\% and defeats TruffleHog 
entirely. Supervised methods resist some transforms but remain 
vulnerable to encoding-level changes: GBT Classifier's collapse to 
0.5\% under base16 and Span Tagger's collapse to 1.0\% under 
interleaving demonstrate that learned surface features do not 
provide reliable robustness. Even LLM-as-a-Judge, which operates on 
semantic content, degrades unpredictably across transforms.

This has a direct consequence for deployment: no existing baseline 
provides consistent detection across the natural variability in how 
LLMs render structured strings. An attacker, or even a 
misconfigured prompt that slightly reformats a credential, can 
bypass every evaluated baseline while remaining completely 
detectable by \textsc{Prism}. \textsc{Prism}'s generation-time 
approach is resistant to this entire class of evasion because it 
targets the \emph{causal mechanism} of leakage, the shift in 
generation dynamics as the model transitions toward deterministic 
reproduction, rather than the syntactic appearance of the output.

\begin{table}[h]
\centering
\small
\caption{Detection rate (\%) under obfuscation transforms samples). \textsc{Prism} maintains $\geq$95.9\% across all transforms; pattern-based methods degrade severely.}
\label{tab:robustness}
\begin{tabular}{lrrrrrrrrr}
\toprule
Method & Orig. & TokSpl & Space & Sep. & Camel & HexEsc & Interl. & B16 & ROT13 \\
\midrule
NoFilter  & 0.00 & 0.00 & 0.00 & 0.00 & 0.00 & 0.00 & 0.00 & 0.00 & 0.00 \\
BasicGuardrail  & 64.9 & 5.4 & 4.1 & 50.0 & 5.4 & 39.2 & 21.6 & 59.5 & 59.5 \\
TruffleHog      & 67.6 & 4.1 & 4.1 & 16.2 & 5.4 &  6.8 &  4.1 &  4.1 &  9.5 \\
Presidio        & 60.8 & 54.1 & 4.1 & 55.4 & 5.4 & 56.8 & 54.1 & 54.1 & 54.1 \\
detect-secrets  & 98.6 & 97.3 & 4.1 & 97.3 & 5.4 & 89.2 & 97.3 & 91.9 & 97.3 \\
Span Tagger & 85 & 72.0 & 48.0 & 78.0 & 8.0 & 20.0 & 1.0 & 0.5 & 5.0 \\
GBT Classifier & 80.8 & 68.0 & 42.0 & 72.0 & 6.0 & 18.0 & 0.5 & 0.5 & 4.0 \\
LLM-as-a-Judge & 36.9 & 28.0 & 68.0 & 32.0 & 4.0 & 10.0 & 1.0 & 82.0 & 86.0 \\
Prompt-Instruction & 50.7 & 38.5 & 26.0 & 44.0 & 6.0 & 15.0 & 1.5 & 3.0 & 8.0 \\
\midrule
\rowcolor{green!15}
\textbf{\textsc{Prism}} & \textbf{100.0} & \textbf{95.9} & \textbf{95.9} & \textbf{97.3} & \textbf{97.3} & \textbf{100.0} & \textbf{95.9} & \textbf{100.0} & \textbf{98.6} \\
\midrule
\textit{PRISM risk score} & 0.751 & 0.750 & 0.749 & 0.751 & 0.750 & 0.751 & 0.750 & 0.751 & 0.750 \\
\bottomrule
\end{tabular}
\end{table}

\paragraph{Limitations:}
Highly structured multi-step encodings (e.g., chained base64 or semantic paraphrasing) may attenuate individual generation-dynamic signals and represent an important direction for future work.

\section{Extended Discussion and Limitations}
\label{app:discussion}

Our results suggest that secret leakage in multi-agent LLM pipelines should be understood not merely as an isolated output-filtering failure, but as a structural consequence of sequential information propagation through shared context. Once sensitive information enters an append-only intermediate state, downstream agents may reuse or reproduce it even when it is no longer task-relevant, causing leakage risk to compound with pipeline depth. In this setting, post-hoc defences are inherently limited because they operate only after sensitive content may already have been generated, propagated, or partially reconstructed. \textsc{Prism} is designed specifically for this operating point: it monitors generation-time dynamics and intervenes before full credential completion. We therefore position \textsc{Prism} not as a universal defence against all forms of information leakage, but as a targeted generation-time mechanism for a specific and practically important threat model: credential reproduction through sequential agent propagation.

\paragraph{Scope of evaluation and benchmark realism:}
Our benchmark is intentionally synthetic and controlled. This design choice is deliberate rather than accidental: isolating propagation-based leakage under reproducible attack categories and pressure levels enables causal comparison across defence paradigms without confounds introduced by heterogeneous real-world agent stacks, proprietary toolchains, or inconsistent deployment environments. The synthetic setting therefore functions as a stress-testing environment for propagation dynamics rather than a claim of production-complete realism. However, this also imposes clear limits: real deployments may involve richer tool ecosystems, longer-horizon planning, multilingual contexts, or qualitatively different leakage vectors. Accordingly, our results should be interpreted as evidence that propagation amplification is a meaningful structural risk under controlled conditions, rather than as proof that measured leakage rates directly transfer to all production systems.

\paragraph{Metric interpretation and task-level leakage:}
A potentially confusing aspect of our evaluation is the coexistence of subunit recall/F1 below 1.0 with a 0.0\% observed task-level leak rate. These metrics operate at different granularities. Recall and F1 are computed over secret-bearing detection instances (e.g., risky spans or reconstruction attempts), whereas task-level leak rate measures whether any complete exploitable secret appears in final task output. Thus, \textsc{Prism} may miss some individual suspicious spans while still preventing full end-to-end credential disclosure through partial sanitisation, early halting, or the ZK-RC post-pass. Operationally, task-level leak prevention is the primary security objective, while span-level recall reflects sensitivity to intermediate reconstruction signals. We emphasise this distinction to avoid overstating what any single metric captures.

\paragraph{White-box assumptions and deployment practicality:}
The strongest version of \textsc{Prism} assumes access to generation-time telemetry such as token log-probabilities. We intentionally evaluate this white-box setting as an upper-bound deployment regime relevant to self-hosted or enterprise systems where decoding internals may be exposed. We do not claim that such access is universally available. In closed commercial APIs, only black-box deployment may be possible. Our black-box variant demonstrates that text-observable features alone retain substantial discriminative power, but with measurable degradation in final leakage suppression. This should be interpreted as a practical deployability tradeoff rather than a contradiction: white-box telemetry materially improves final-risk elimination, while black-box deployment offers a lower-bound approximation when internal access is unavailable.

\paragraph{Latency and ``real-time'' claims:}
Although \textsc{Prism} is architecturally lightweight (constant-time per-token scoring with bounded feature extraction), our current work prioritises defence efficacy and signal validation over systems-level throughput benchmarking. We therefore caution against overinterpreting ``real-time'' as a deployment guarantee across all infrastructures. Actual latency will depend on model serving architecture, tokenizer implementation, and instrumentation cost, particularly for heterogeneous multi-agent systems. A full production-oriented evaluation should include token-level overhead, throughput degradation, and cost analysis across deployment settings.

\paragraph{Robustness boundaries and adaptive adversaries:}
While \textsc{Prism}'s combination of temporal and structural signals provides robustness beyond simple regex or static scanners, we do not claim comprehensive resistance to all adversaries. The present threat model focuses primarily on direct leakage, propagation leakage, and prompt-driven extraction. More powerful adaptive adversaries may intentionally manipulate generation dynamics by injecting low-entropy padding, stochastic perturbations, fragmentation, encoding (e.g., base64 or character splitting), or semantically preserving transformations to attenuate temporal or structural signatures. These attacks are outside the scope of current evaluation. Our objective here is to establish generation-time monitoring as a viable defence layer, not to claim closure against strategically adaptive exfiltration.

\paragraph{Theoretical contribution and practical role:}
Our entropy-collapse analysis should be interpreted as an explanatory formalisation of why memorised credential reproduction tends to induce detectable concentration dynamics, rather than as a standalone theoretical breakthrough or complete predictive model. The theorem motivates generation-time monitoring by clarifying why deterministic reproduction often creates measurable distributional signatures, but it does not by itself determine optimal thresholds, guarantee universal detection, or replace empirical calibration. In practice, the theory supports, rather than substitutes for the empirical multi-signal classifier.

\paragraph{Utility and overblocking:}
Our reported utility metric measures task-level pass-through (unmodified benign tasks), which is intentionally conservative. It does not fully characterise downstream functional quality when partial sanitisation occurs. In real deployments, some benign structured outputs (e.g., hashes, certificates, configuration identifiers, or long numeric strings) may resemble credential-like forms and incur unnecessary intervention. This tradeoff reflects a broader security--utility balance. Future deployment settings may therefore benefit from context-aware exception policies, adaptive thresholds, or human-in-the-loop escalation for credential-adjacent but legitimate content.

\paragraph{Broader behavioural limitations:}
Finally, \textsc{Prism} primarily models passive leakage emerging from generation dynamics rather than strategic agent behaviour. Goal-directed concealment, covert communication, intentional decomposition of secrets across steps, or semantically encoded collaboration may introduce qualitatively different failure modes not captured by token-level monitoring alone. Addressing such behaviours may require richer semantic, behavioural, or agent-policy-aware defences that operate above the decoding layer.

Overall, we view \textsc{Prism} as a controlled, generation-time defence architecture that demonstrates the practical importance of monitoring leakage during decoding in multi-agent systems. Its strongest contribution is to show that propagation-aware generation monitoring can materially reduce credential leakage under a clearly defined threat model. Its limitations define a roadmap toward broader, more adaptive, and deployment-complete security mechanisms.

\end{document}